\documentclass{article}

\usepackage[main, final]{neurips_2026}
\usepackage{wrapfig}

\usepackage[utf8]{inputenc} 
\usepackage[T1]{fontenc}    
\usepackage{hyperref}       
\usepackage{url}            
\usepackage{booktabs}       
\usepackage{amsfonts}       
\usepackage{nicefrac}       
\usepackage{microtype}      
\usepackage{xcolor}         
\usepackage{amsmath}
\usepackage{amsthm}
\usepackage{graphicx}
\usepackage{algorithm}
\usepackage{algpseudocode}

\usepackage{siunitx}
\usepackage{threeparttable}

\usepackage{lineno}

\usepackage{microtype}
\usepackage{graphicx}
\usepackage{subfigure}
\usepackage{booktabs} 
\usepackage{multirow} 
\usepackage{hyperref}

\usepackage[table]{xcolor}

\usepackage{mathtools}
\usepackage{amsthm}
\usepackage{bm}
\usepackage{algorithm}
\usepackage{algorithmicx,algpseudocode}
\usepackage[capitalize,noabbrev]{cleveref}

\newcommand{\RR}{\mathbb{R}}
\newcommand{\HH}{\mathcal{H}}
\newcommand{\inner}[2]{\langle #1,\, #2 \rangle}
\newcommand{\norm}[1]{\lVert #1 \rVert}
\DeclareMathOperator*{\argmax}{arg\,max}

\definecolor{darkblue}{rgb}{0, 0, 0.5}
\hypersetup{colorlinks=true, citecolor=darkblue, linkcolor=darkblue, urlcolor=darkblue}

\newtheorem{theorem}{Theorem}[section]

\newtheorem{corollary}[theorem]{Corollary}
\newtheorem{proposition}[theorem]{Proposition}

\newtheorem{remark}[theorem]{Remark}

\usepackage{titletoc}

\newcommand\DoToC{%
  \startcontents
  \printcontents{}{1}{\textbf{Table of Contents}\vskip3pt\hrule\vskip5pt}
  \vskip3pt\hrule\vskip5pt
}

\newcommand\blfootnote[1]{%
  \begingroup
  \renewcommand\thefootnote{}\footnote{#1}%
  \addtocounter{footnote}{-1}%
  \endgroup
}

\usepackage[font=small]{caption}

\title{Kernelized Activation Steering}

\author{%
  Laziz U. Abdullaev$^{\ast}$ \\
  National University of Singapore\\
  \texttt{laziz.abdullaev@u.nus.edu}\\
  \And
  Minh-Hieu Pham$^{\ast}$\\
  Hanoi University of Science and Technology\\
  \texttt{hieu.pm220062@sis.hust.edu.vn}\\
  \And
  Bach Do\\
  Hanoi University of Science and Technology\\
  \texttt{bach.dtg225472@sis.hust.edu.vn}\\
  \And
  Khoat Than$^{\dagger}$\\
  Hanoi University of Science and Technology\\
  \texttt{khoattq@soict.hust.edu.vn}\\
  \And
  Tan M. Nguyen$^{\dagger}$\\
  National University of Singapore\\
  \texttt{tanmn@nus.edu.sg}
}

\begin{document}

\maketitle

\begin{abstract}
Activation steering provides a simple, training-free mechanism for controlling attributes of generative models such as sentiment, style, and helpfulness. However, standard approaches such as Difference-in-Means apply a single input-independent steering vector across all activations, limiting expressivity and ignoring the local geometry of the activation space.
We propose Kernelized Activation Steering (KAS), a unifying framework that lifts activation steering into a reproducing kernel Hilbert space. KAS formulates steering as an optimization problem expressed purely via kernel evaluations, yielding an implicit, activation-dependent steering score without constructing explicit feature maps.
Unlike DiM, KAS induces locally adaptive steering: each activation is modified according to its relative position with respect to source and target reference sets, producing a nonlinear steering field over the representation space. Importantly, DiM is recovered as a special case under a linear kernel, while richer kernels enable geometry-aware interventions.
Across standard activation steering tasks, including jailbreaking LLMs and image style control, KAS outperforms or is on par with the existing methods.

\blfootnote{$^{\ast}$Co-first authors. $^{\dagger}$Co-last authors. Correspondence to: laziz.abdullaev@u.nus.edu \& tanmn@nus.edu.sg.}

\blfootnote{Our code is available at \text{https://github.com/KleepyHunter/kernel\_steering}}
\end{abstract}

\section{Introduction}
\label{sec:intro}
 
Large language models (LLMs) are now known to encode rich semantic concepts such as sentiment, refusal, truthfulness and persona as identifiable structures in their internal representations \citep{tigges2024language, marks2024the, arditi2024refusal}. A central problem in alignment and controllability is to manipulate these structures \emph{at inference time}, in order to alter behaviour without the cost or fragility of full fine-tuning. \emph{Activation steering} \citep{turner2024steeringlanguagemodelsactivation, rimsky2024steering, arditi2024refusal,vu2025angular} has emerged as a particularly practical instrument for this purpose: it modifies a single layer's activations using a precomputed direction and requires no parameter updates.
 
The dominant instantiation of this idea is \emph{Difference-in-Means} (DiM) \citep{rimsky2024steering, turner2024steeringlanguagemodelsactivation, arditi2024refusal,nguyen2026activation}, which estimates the steering direction as the difference between mean activations on a \emph{source} dataset (behaviour to suppress) and a \emph{target} dataset (behaviour to elicit). DiM derives both its appeal and its limitations from the \emph{linear representation hypothesis} (LRH) \citep{park2025linear, mikolov2013linguistic, arora2018linear, elhage2022superposition} — the assumption that high-level concepts are encoded as linear directions in activation space. While the LRH has proven surprisingly robust for many tasks, mounting evidence highlights pervasive polysemanticity and curved structure in transformer representations \citep{engels2025languagemodelfeaturesonedimensionally, elhage2022superposition}, where features cluster along nonlinear manifolds that no single global translation can faithfully capture. As a consequence, DiM-style interventions can be either ineffective on out-of-distribution activations or destructive of unrelated capabilities once the steering coefficient is increased to compensate.
 
We propose \textbf{Kernelized Activation Steering (KAS)}, a framework that retains the simplicity of DiM while relaxing its linear constraint via the kernel trick. We lift activations into a reproducing kernel Hilbert space (RKHS) $\HH$, define the feature-space steering direction as the lifted DiM vector, and derive an \emph{implicit} steering score that depends on the input only through kernel evaluations against the reference sets. Steering then proceeds by gradient ascent on this score with a proximal regulariser that bounds deviation from the original activation. Because kernel gradients themselves depend on the current activation, the resulting update is locally adaptive: KAS attracts activations toward target reference points and repels them from source reference points, producing a nonlinear steering field (see Figure \ref{fig:dim-vs-kas}). Crucially, KAS recovers DiM exactly under a linear kernel, situating it as a strict generalization rather than an alternative.
 
Overall, our contributions are threefold:
\begin{enumerate}
    \item \textit{A nonlinear generalization of activation steering.} We formulate activation steering as optimization in an RKHS and derive an implicit steering score that requires only kernel evaluations between the input and the source/target reference sets (\S\ref{sec:method}). Furthermore, DiM is recovered exactly under the linear kernel (Remark~\ref{rem:linear-recovery}), unifying classical kernel methods with modern activation-level interventions.
    \item \textit{A theoretical bridge between activation steering and kernel two-sample testing.} We show that the KAS score is the empirical witness function associated with maximum mean discrepancy, connecting activation steering to kernel two-sample testing and explaining how characteristic kernels can detect distributional differences that linear mean-based steering misses.
    \item \textit{A locally adaptive RBF instantiation.} We instantiate the framework with the Gaussian RBF kernel, yielding a closed-form, activation-dependent gradient update (\S\ref{sec:rbf}).
\end{enumerate}
 
We evaluate KAS on jailbreak steering across three instruction-tuned LLMs, CAA-style behavioral steering, and text-to-image style control. KAS improves jailbreak attack success over ActAdd on Gemma and Qwen while matching it on Llama, achieves the strongest effects on four CAA behaviors with non-null steering responses, and improves style induction over Mean-AcT across most tested intervention strengths while preserving prompt alignment. Additional ablations examine kernel approximation and steering from a single contrastive pair (\S\ref{sec:experiments}).

\begin{figure}
    \centering
    \includegraphics[width=\linewidth]{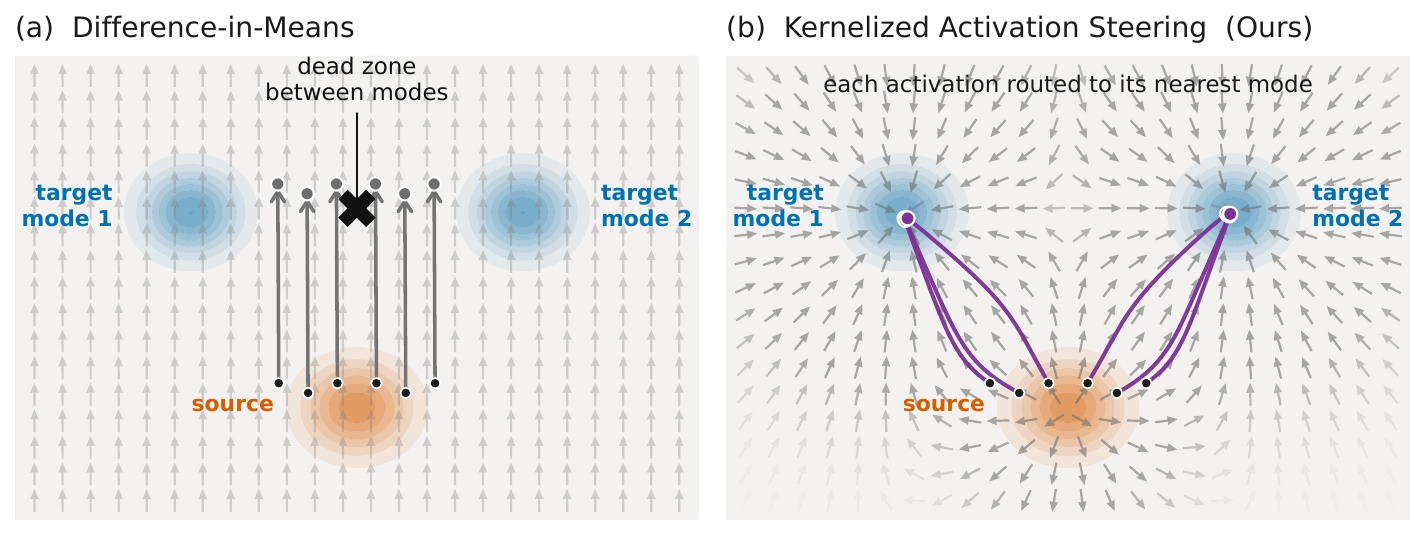}
    \vspace{-1em}
    \caption{An illustrative comparison between standard Difference-in-Means (DiM) and Kernelized Activation Steering (KAS). Source activations (orange) lie between two target modes (blue).  \textbf{(a)} DiM applies a single global translation $\bm{v} = \mu_B - \mu_A$, sending every activation into the saddle between the modes, a region with no semantic content under either target behaviour. No choice of steering strength $\alpha$ recovers the modes. \textbf{(b)} Kernelized activation steering (KAS) produces a data-adaptive vector field that routes each activation to its nearest  mode.}
    \label{fig:dim-vs-kas}
     \vspace{-1em}
\end{figure}
 
\section{Preliminaries}
\label{sec:prelim}
 
\textbf{Notation.}
We write $[n] := \{1, \ldots, n\}$ for any positive integer $n$. For a decoder-only transformer \citep{vaswani2017attention}, we denote by $\bm{x}^{(\ell)} \in \RR^d$ the residual-stream activation at layer $\ell \in [L]$ for a given input prompt.
 
\textbf{Activation Steering via Difference-in-Means.}
Activation steering modifies a model's behaviour at inference time by directly intervening on internal activations, without updating parameters. Given a contrastive pair of prompt sets—a \emph{source} set $A$ exhibiting behaviour to suppress and a \emph{target} set $B$ exhibiting behaviour to elicit—the difference-in-means (DiM) method \citep{rimsky2024steering, turner2024steeringlanguagemodelsactivation, arditi2024refusal, venhoff2025understanding,nguyen2026activation} extracts a steering vector by contrasting their mean activations:
\begin{equation}\label{eq:dim}
    \bm{v} = \frac{1}{n_B}\sum_{j \in [n_B]} \bm{b}_j - \frac{1}{n_A}\sum_{i \in [n_A]} \bm{a}_i,
\end{equation}
where $\{\bm{a}_i\}_{i \in [n_A]}$ and $\{\bm{b}_j\}_{j \in [n_B]}$ are the layer-$\ell^{*}$ activations of prompts in $A$ and $B$. At inference, an incoming activation $\bm{x}$ is steered by the additive update
\begin{equation}\label{eq:dim-update}
    \bm{x} \;\leftarrow\; \bm{x} + \alpha\, \bm{v},
\end{equation}
where $\alpha \in \RR$ controls the intervention strength. The intervention layer $\ell^{*}$ is typically chosen on a held-out validation set \citep{arditi2024refusal}.
 
A defining property of DiM is that the steering direction $\bm{v}$ is \emph{global} and \emph{input-independent}: every activation receives the same translation, regardless of where it sits in representation space. Our work generalises this construction to allow activation-dependent steering through a kernelised formulation.
 
\textbf{Kernel Methods and Reproducing Kernel Hilbert Spaces.}
A function $k : \mathcal{X} \times \mathcal{X} \to \RR$ is a \emph{positive semi-definite (PSD) kernel} if, for any finite set $\{\bm{x}_i\}_{i=1}^n \subset \mathcal{X}$, the Gram matrix $K \in \RR^{n \times n}$ with entries $K_{ij} = k(\bm{x}_i, \bm{x}_j)$ is positive semi-definite. By the Moore–Aronszajn theorem \citep{aronszajn1950theory, berlinet2004reproducing}, every PSD kernel uniquely defines a reproducing kernel Hilbert space (RKHS) $\HH$ together with a feature map $\varphi : \mathcal{X} \to \HH$ satisfying $k(\bm{x}, \bm{y}) \;=\; \inner{\varphi(\bm{x})}{\varphi(\bm{y})}_{\HH}$ for all $\bm{x}, \bm{y} \in \mathcal{X}$.
The \emph{kernel trick} states that any algorithm expressible solely through inner products in $\HH$ can be implemented using only kernel evaluations on $\mathcal{X}$, without materializing $\varphi$. This is what enables us to operate in a high- or infinite-dimensional feature space at the cost of $\mathcal{O}(n)$ kernel evaluations per query \citep{scholkopf2002learning, hofmann2008kernel}. Kernel methods are usually favorable as it is believed that most nonlinear systems behave approximately linearly in higher, possibly infinite-dimensional feature spaces \citep{cover1965geometrical, bouvrie2017kernel}.
 
\section{Related Work}
\label{sec:related}
 
Activation-level interventions on language models have been studied under several names, including activation addition \citep{turner2024steeringlanguagemodelsactivation}, representation engineering \citep{zou2023representation}, contrastive activation addition \citep{rimsky2024steering}, and inference-time intervention \citep{li2023inferencetime}. Most of these methods extract a single steering direction per concept via difference-in-means, linear probing, or PCA on contrastive activations \citep{alain2016understanding, li2023inferencetime, zou2023representation, teo2025the,vu2025angular,nguyen2026activation}. A separate line of work uses sparse autoencoders to identify steerable features \citep{templeton2024scaling, obrien2024steering,bayat2025steering, kantamneni2025are, xie2025comparative}. However, most of these approaches apply a fixed translation in activation space and thus inherit the limitations of input-independent steering. To address this issue, there is a growing body of work dedicated to designing input-dependent steering mechanisms. K-steer \citep{oozeer2025beyond} trains a nonlinear attribute classifier on hidden activations and computes steering directions via gradients of that classifier while SAS \citep{bayat2025steering} relies on pretrained SAEs as nonlinear feature extractors. Another line of research models activation steering via Optimal Transport \citep{rodriguez2024meanact} through which numerous input-dependent affine \citep{singh2024representation} and nonlinear \citep{abdullaev2026concept} steering methods have been proposed. Perhaps the closest to our method is Curveball steering \citep{raval2026curveball}, which uses kernel PCA for nonlinear dimensionality reduction and performs steering in the reduced space while \citep{saillenfest2025nonlinear} employs a similar technique for the concept erasure problem \citep{ravfogel2020null, belrose2023leace, chowdhury2025fundamental}. More broadly, kernel methods have been used for distribution matching in generative modeling \citep{li2015generative}, particle-based variational inference \citep{liu2016stein}, mode seeking in diffusion-based image restoration \citep{hu2025kernel}, and deep learning architecture design \citep{jacot2018neural,tsai-etal-2019-transformer,nguyen2022fourierformer,ataee2023max,han2023designing,chen2023primal,nguyen2023a,tarzanagh2023transformers,nielsen2024elliptical,bui2025revisiting,shen2026understanding}, providing related distributional perspectives. Our training-free framework, Kernelized Activation Steering (KAS), instead derives an activation-dependent intervention from an interpretable, kernel-based optimization objective and recovers DiM exactly as a special case under a linear kernel, along with a relation to hypothesis testing techniques \citep{gretton2012kernel, muandet2017kernel} providing a complementary theoretical lens for activation steering design.
 
\section{Kernelized Activation Steering}
\label{sec:method}
 
Let $\{\bm{a}_i\}_{i \in [n_A]}$ and $\{\bm{b}_j\}_{j \in [n_B]}$ denote the layer-$\ell^{*}$ activations of prompts drawn from a source set $A$ and target set $B$, respectively. The DiM steering vector \eqref{eq:dim} is effective when $A$ and $B$ are well separated by a linear hyperplane in activation space, but totally disregards local structures of the activation distributions and possible nonlinearities in the decision boundary. To approach these issues properly, we shall derive DiM activation steering as a special case of a more general, principled operation in the following section.
 
\subsection{Lifting Difference-in-Means into an RKHS}
\label{sec:lift}
 
Let $\varphi : \RR^d \to \HH$ be the feature map associated with a PSD kernel $k$, where $\HH$ may be infinite-dimensional. The natural lift of \eqref{eq:dim} into $\HH$ is
\begin{equation}\label{eq:phi-dim}
    \bm{v}_{\varphi} \;=\; \frac{1}{n_B}\sum_{j \in [n_B]} \varphi(\bm{b}_j) \;-\; \frac{1}{n_A}\sum_{i \in [n_A]} \varphi(\bm{a}_i).
\end{equation}
For sufficiently expressive kernels, for example, universal kernels such as the Gaussian RBF \citep{steinwart2008support}, the contrastive features of the lifted samples $\{\varphi(\bm{a}_i)\}$ and $\{\varphi(\bm{b}_j)\}$ are expected to be captured well in $\HH$, motivating $\bm{v}_{\varphi}$ as a steering direction in the feature space. Materialising $\bm{v}_{\varphi}$ explicitly is, however, infeasible when $\HH$ is very high or infinite-dimensional. We circumvent this by expressing the entire steering objective in terms of kernel evaluations as we discuss next.
 
Ideally, for an activation $\bm{x}_0 \in \RR^d$, we would like the steered activation $\bm{x}$ to satisfy the equation $\varphi(\bm{x}) - \varphi(\bm{x}_0) = \eta \bm{v}_{\varphi}$ for some $\eta > 0$, which represents that the feature-space steering direction is the same as $\bm v_\varphi$. Motivated by this, we design the feature-space steering mechanism to increase the dot product similarity $\inner{ \varphi(\bm{x}) - \varphi(\bm{x}_0)}{\bm{v}_{\varphi}}$ with respect to $\bm x$. As the contribution of $\varphi(\bm x_0)$ to this objective is independent of $\bm x$, we define the following \emph{implicit steering score}
\begin{equation}\label{eq:pi-def}
    \Pi(\bm{x}) \;:=\; \inner{\varphi(\bm{x})}{\bm{v}_{\varphi}}_{\HH}.
\end{equation}
By bilinearity of the inner product and the kernel-trick identity $\inner{\varphi(\bm{x})}{\varphi(\bm{y})}_{\HH} = k(\bm{x}, \bm{y})$,
\begin{align}\label{eq:pi-kernel}
    \Pi(\bm{x}) &= \inner{\varphi(\bm{x})}{\frac{1}{n_B}\sum_{j \in [n_B]} \varphi(\bm{b}_j) \;-\; \frac{1}{n_A}\sum_{i \in [n_A]} \varphi(\bm{a}_i)} \nonumber \\
    &= \frac{1}{n_B}\sum_{j \in [n_B]} k(\bm{x}, \bm{b}_j) \;-\; \frac{1}{n_A}\sum_{i \in [n_A]} k(\bm{x}, \bm{a}_i),
\end{align}
which depends only on kernel evaluations between $\bm{x}$ and the reference sets. Intuitively, $\Pi(\bm{x})$ measures the alignment of $\varphi(\bm{x})$ with the lifted DiM direction $\bm{v}_{\varphi}$. Therefore, increasing $\Pi$ pushes the feature representation of $\bm{x}$ toward the target set and away from that of the source set.
 
\subsection{Steering as an Optimization Problem}
\label{sec:opt}
 
Given an initial activation $\bm{x}_0 \in \RR^d$, we cast kernelized steering as the regularized maximization problem given by
\begin{equation}\label{eq:kas-opt}
    \bm{x}^{*} \;=\; \argmax_{\bm{x} \in \RR^d}\;
    F(\bm{x}) \,, \qquad
    F(\bm{x}) \;:=\; \Pi(\bm{x}) \;-\; \frac{\lambda}{2}\, \norm{\bm{x} - \bm{x}_0}^2,
\end{equation}
where $\lambda \ge 0$ is a proximal coefficient that optionally prevents the steered activation from drifting too far from $\bm{x}_0$ and thereby disrupting unrelated capabilities of the model. When $k$ is differentiable, we approximate $\bm{x}^{*}$ by gradient ascent on $F$:
\begin{align}\label{eq:kas-grad}
    \bm{x} \;\gets\; \bm{x} + \eta \!\left[
        \frac{1}{n_B}\sum_{j \in [n_B]} \nabla_{\bm{x}}\, k(\bm{x}, \bm{b}_j)
        - \frac{1}{n_A}\sum_{i \in [n_A]} \nabla_{\bm{x}}\, k(\bm{x}, \bm{a}_i)
        - \lambda\, (\bm{x} - \bm{x}_0)
    \right],
\end{align}
where $\eta > 0$ is the step size. Crucially, the update depends on $\bm{x}$ through the kernel gradients, so the effective steering direction \emph{varies with the current activation} as opposed to DiM, which applies a fixed translation regardless of input. Also, note that we set $\lambda = 0$ for most of our experiments since we usually take only one gradient ascent step on the objective as our steering mechanism due to inference-time computational constraints.

\begin{remark}[DiM as a special case]
\label{rem:linear-recovery}
For the linear kernel $k(\bm{x}, \bm{y}) = \bm{x}^\top \bm{y}$, we have $\nabla_{\bm{x}}\, k(\bm{x}, \bm{y}) = \bm{y}$, so \eqref{eq:kas-grad} reduces, in the unregularized case, to
\begin{equation*}
    \bm{x} \;\gets\; \bm{x} + \eta\!\left[
        \frac{1}{n_B}\sum_{j \in [n_B]} \bm{b}_j
        \;-\; \frac{1}{n_A}\sum_{i \in [n_A]} \bm{a}_i
    \right] \;=\; \bm{x} + \eta\, \bm{v},
\end{equation*}
which recovers DiM steering with strength $\alpha = \eta$. Thus KAS strictly generalizes DiM, with general kernels yielding activation-dependent updates that DiM cannot express.
\end{remark}

\begin{remark}[Normalized implementation]
\label{rem:normalized-update}
Equation~\eqref{eq:kas-grad} presents the raw-gradient update to make the connection to the optimization objective and DiM transparent. In Algorithm~\ref{alg:rbf} and all experiments, we instead use the normalized update $\bm{x}_{t+1}=\bm{x}_t+\eta\,\nabla F(\bm{x}_t)/\norm{\nabla F(\bm{x}_t)}$, whenever $\nabla F(\bm{x}_t)\neq 0$. The full gradient $\nabla F(\bm{x}_t)=\nabla\Pi(\bm{x}_t)-\lambda(\bm{x}_t-\bm{x}_0)$ is formed before normalization, so $\eta$ represents an absolute step length in activation space. Under the linear kernel with $\lambda=0$, this becomes $\bm{x}_{t+1}=\bm{x}_t+\eta\,\bm{v}/\norm{\bm{v}}$, which recovers the same one-parameter family of DiM interventions $\bm{x}_t+\alpha\bm{v}$ under the correspondence $\eta=\alpha\norm{\bm{v}}$.
\end{remark}

\textbf{Connection to maximum mean discrepancy.}
Maximum mean discrepancy (MMD) \citep{gretton2012kernel} measures the difference between two distributions through their mean embeddings in an RKHS. For source and target activation distributions $q_A$ and $q_B$, respectively, define $\mu_A := \mathbb{E}_{\bm{a}\sim q_A}[\varphi(\bm{a})]$ and $\mu_B := \mathbb{E}_{\bm{b}\sim q_B}[\varphi(\bm{b})]$. Their MMD is
$$
    \mathrm{MMD}[\HH,q_B,q_A]
    = \|\mu_B-\mu_A\|_{\HH}
    = \sup_{\|f\|_{\HH}\leq 1}
      \left(\mathbb{E}_{q_B}[f]-\mathbb{E}_{q_A}[f]\right).
$$
The maximizing function is known as the \emph{MMD witness function}. It assigns larger values to regions of the input space that are more characteristic of $q_B$ than of $q_A$, with the notion of similarity determined by the kernel. This construction closely mirrors KAS in the sense that the lifted steering direction
$\bm{v}_{\varphi}$ in \eqref{eq:phi-dim} is the empirical difference between the two kernel mean embeddings, and evaluating it at an activation $\bm{x}$ produces the score $\Pi(\bm{x})$. Thus, whereas a conventional two-sample test aggregates the witness into a scalar distributional discrepancy, KAS uses its gradient as an intervention, moving individual activations toward regions favored by the target distribution and away from regions favored by the source distribution. The following proposition formalizes this equivalence.


\begin{proposition}[KAS as MMD witness ascent]
\label{prop:mmd-witness}
Let $q_A$ and $q_B$ denote the source and target activation distributions, with
mean embeddings $\mu_A, \mu_B \in \HH$. The population steering score $\Pi^{\infty}(\bm{x})
    \;:=\; \mathbb{E}_{\bm{b}\sim q_B}[k(\bm{x},\bm{b})]
        - \mathbb{E}_{\bm{a}\sim q_A}[k(\bm{x},\bm{a})]$
is the unnormalized maximum mean discrepancy (MMD) witness function for the
two-sample problem $q_B$ vs.\ $q_A$ \citep{gretton2012kernel}. Concretely,
$$
    \Pi^{\infty}(\bm{x}) \;=\; \mathrm{MMD}[\HH,q_B,q_A] \cdot f^{*}(\bm{x}),
$$
where $f^{*} \in \HH$, $\|f^{*}\|_{\HH}=1$ is a witness function that attains
$\sup_{\|f\|_{\HH}\le 1}[\mathbb{E}_B f - \mathbb{E}_A f]$.
The empirical KAS score $\Pi$ in \eqref{eq:pi-kernel} is the standard plug-in
estimator of $\Pi^{\infty}$, and Algorithm~\ref{alg:rbf} performs gradient ascent
on this empirical witness.
\end{proposition}

\label{rem:mmd-consequences}
Identifying $\Pi$ with the MMD witness imports a couple of useful properties from
kernel two-sample testing.
\textbf{(i) Characteristic kernels and signal existence.}
When $k$ is \emph{characteristic} on $\RR^d$ (true for the Gaussian RBF and
Laplace kernels \citep{gretton2012kernel,sriperumbudur2010hilbert}),
$\Pi^{\infty}\equiv 0$ if and only if $q_A=q_B$. In contrast, a linear kernel
yields $\Pi^{\infty}\equiv 0$ iff the \emph{means} of the two distributions
agree, so DiM has no population signal whenever source and target
activations share a centroid but differ in higher moments (e.g. multimodal
targets). This formalizes DiM's principal failure mode and motivates the use
of nonlinear kernels in KAS.
\textbf{(ii) Bandwidth selection.}
The MMD median heuristic, setting $\sigma$ to the median pairwise Euclidean
distance in $A\cup B$, gives a parameter-free, well-studied default for
the RBF bandwidth \citep{gretton2012kernel}, which we will also adopt for our experiments.
 
\begin{figure*}[t]
\centering
\begin{minipage}[t]{0.46\textwidth}
\begin{algorithm}[H]
\caption{DiM Steering \citep{arditi2024refusal}}
\label{alg:dim}
\begin{algorithmic}[1]
\Require Activation $\bm{x}_0 \in \RR^d$; source set $A = \{\bm{a}_i\}_{i=1}^{n_A}$; target set $B = \{\bm{b}_j\}_{j=1}^{n_B}$; strength $\alpha$
\Statex
\State \textbf{Compute steering vector:}
\Statex \quad $\bm{v} \gets \frac{1}{n_B}\sum_{j=1}^{n_B} \bm{b}_j - \frac{1}{n_A}\sum_{i=1}^{n_A} \bm{a}_i$
\Statex
\State \textbf{Update activation:}
\Statex \quad $\bm{x} \gets \bm{x}_0 + \alpha\, \bm{v}$
\Statex
\State \Return steered activation $\bm{x}$
\end{algorithmic}
\end{algorithm}
\end{minipage}
\hfill
\begin{minipage}[t]{0.52\textwidth}
\begin{algorithm}[H]
\caption{Kernelized Activation Steering}
\label{alg:rbf}
\begin{algorithmic}[1]
\Require Activation $\bm{x}_0$; source $A = \{\bm{a}_i\}_{i=1}^{n_A}$; target $B = \{\bm{b}_j\}_{j=1}^{n_B}$; bandwidth $\sigma$; step size $\eta$; regulariser $\lambda$; iterations $T$
\Statex
\State Initialise $\bm{x} \gets \bm{x}_0$
\For{$t = 1, \ldots, T$}
    \State \textbf{Kernel-gradient direction:}
    \Statex \quad $ \bm{g} \gets \left[\frac{1}{n_B}\sum_{j=1}^{n_B} \nabla_{\bm{x}}k(\bm{x}, \bm{b}_j) \right.$
    \Statex \quad \phantom{$\bm{g} \gets$} $\left. \quad - \frac{1}{n_A}\sum_{i=1}^{n_A} \nabla_{\bm{x}}k(\bm{x}, \bm{a}_i)\right]$
    \State $\bm{g} \gets \bm{g}-\lambda(\bm{x}-\bm{x}_0)$
    \State $\bm{x} \gets \bm{x}+\eta\,\bm{g}/\norm{\bm{g}}$
\EndFor
\State \Return steered activation $\bm{x}$
\end{algorithmic}
\end{algorithm}
\end{minipage}
\caption{Side-by-side comparison of global linear steering (DiM, Algorithm~\ref{alg:dim}) and locally adaptive RBF-kernelised activation steering (Algorithm~\ref{alg:rbf}). DiM applies the same translation to every activation; KAS adapts the update direction to the local geometry around $\bm{x}$.}
\label{fig:algos}
\end{figure*}
 
\subsection{RBF Kernelized Activation Steering}
\label{sec:rbf}
 
While prior activation-steering methods are restricted to linear operations in activation space, our framework supports arbitrary differentiable kernels. We now develop the case of the Gaussian RBF kernel, which we treat as our default instantiation unless stated otherwise. We detail all kernels considered in this paper and their gradients in Appendix \ref{app:kernel-fields}.
 
We adopt the radial basis function kernel
\begin{equation}\label{eq:rbf}
    k(\bm{x}, \bm{y})
    \;=\; \exp\!\left(-\frac{\norm{\bm{x}-\bm{y}}^2}{2\sigma^2}\right),
\end{equation}
where $\sigma > 0$ is the bandwidth. Unlike the linear kernel, \eqref{eq:rbf} induces a locally adaptive geometry: similarity decays smoothly with Euclidean distance, so reference points only exert influence within an effective radius $\sim \sigma$. The gradient with respect to $\bm{x}$ admits the closed form
\begin{equation}\label{eq:rbf-grad}
    \nabla_{\bm{x}}\, k(\bm{x}, \bm{y})
    \;=\; \frac{1}{\sigma^2}\, k(\bm{x}, \bm{y})\,(\bm{y} - \bm{x}).
\end{equation}
Substituting \eqref{eq:rbf-grad} into the update rule \eqref{eq:kas-grad} yields a nonlinear, data-adaptive steering field: each reference activation contributes a force aligned with $(\bm{y} - \bm{x})$ and weighted by its RBF similarity to $\bm{x}$, producing localized attraction toward target examples and localized repulsion from source examples. We summarize the resulting procedure in Algorithm~\ref{alg:rbf}. This mechanism cannot be reduced to a single fixed steering vector as the effective direction at $\bm{x}$ is a similarity-weighted resultant of $n_A + n_B$ pairwise forces, and changes as $\bm{x}$ traverses representation space. 

\textbf{One-Step Ascent Guarantees.}
We now give sufficient step-size conditions under which the normalized-gradient update used in Algorithm~\ref{alg:rbf} strictly improves the KAS objective $F$. The result follows from the standard ascent lemma, with an explicit kernel-dependent bound for the RBF case in Corollary~\ref{cor:rbf-stepsize}. All proofs are deferred to Appendix~\ref{app:proofs}.

\begin{theorem}[One-Step Ascent]
\label{thm:one-step-ascent}
Let $F(\bm{x})=\Pi(\bm{x})-\tfrac{\lambda}{2}\norm{\bm{x}-\bm{x}_0}^2$, where $\Pi$ is induced by a twice-differentiable kernel $k$ via \eqref{eq:pi-kernel}. Assume $\norm{\nabla_{\bm{x}}^2 k(\bm{x},\bm{y})}_{\mathrm{op}}\le H$ for all $\bm{x},\bm{y}\in\RR^d$. For any iterate $\bm{x}_t$ with $\nabla F(\bm{x}_t)\neq 0$, consider the normalized-gradient update $\bm{x}_{t+1}=\bm{x}_t+\eta\,\nabla F(\bm{x}_t)/\norm{\nabla F(\bm{x}_t)}$. Then $F(\bm{x}_{t+1})>F(\bm{x}_t)$ whenever $0<\eta<\norm{\nabla F(\bm{x}_t)}/(H+\lambda/2)$.
\end{theorem}

For comparison, the raw-gradient update $\bm{x}_{t+1}=\bm{x}_t+\eta\nabla F(\bm{x}_t)$ strictly increases $F$ whenever $0<\eta<(H+\lambda/2)^{-1}$.

\begin{corollary}[One-step ascent for RBF steering]
\label{cor:rbf-stepsize}
Suppose the RBF kernel \eqref{eq:rbf} is used and that the reference and iterated activations lie in a Euclidean ball of radius $R$. Then the normalized update in Theorem~\ref{thm:one-step-ascent} strictly increases $F$ whenever $0<\eta<\frac{\norm{\nabla F(\bm{x}_t)}}{\frac{4R^2}{\sigma^4}+\frac{1}{\sigma^2}+\frac{\lambda}{2}}$.
\end{corollary}

It is worth noting that in practice, we only run a few gradient steps (usually a single step is sufficient), not until convergence. We also tune $\eta$ as a hyperparameter because these sufficient bounds can be conservative. If $\nabla F(\bm{x}_t)=0$, the normalized direction is undefined and the algorithm terminates.

\section{Experiments}
\label{sec:experiments}

We evaluate KAS across three settings that exercise different aspects of activation steering: (i) \textit{CAA-style contrastive steering} on multiple-choice question answering, which isolates targeted behavioral control on a fixed answer-token position; (ii)  \textit{jailbreak steering} on instruction-tuned LLMs, which trades off attack success against general-capability retention; and (iii) \textit{text-to-image style control} on a diffusion model, which probes whether KAS generalizes beyond language modeling. In Section \ref{sec:caa-steering}, we comprehensively evaluate KAS with five different kernel families with extensive hyperparameter search grid for all reported methods. The rest of the main experiments are then limited to standard RBF kernel only as our default testbed in this paper. Sections~\ref{sec:kernel-approximation} and~\ref{sec:ablations} examine kernel approximation and data efficiency, respectively. All experiments are conducted on a single NVIDIA H100 96GB GPU.


\begin{table}[t]
  \centering
  \caption{Held-out CAA target-answer probability and TinyBenchmarks
  utility (\%, higher is better). Bold and underline indicate the two
  highest distinct scores among steered methods in each row.}
  \label{tab:caa-full-depth}
  \tiny
  \setlength{\tabcolsep}{4pt}
  \begin{tabular}{@{}lccccccccl@{}}
    \toprule
    & & & & \multicolumn{5}{c}{Fixed-family KAS}
      & \multicolumn{1}{c}{Validation-selected KAS} \\
    \cmidrule(lr){5-9}\cmidrule(l){10-10}
    Behavior & Unsteered & DiM & CHaRS
      & Gaussian RBF & Quadratic & Laplacian & Cosine RBF
      & Rational quad. & Score (kernel) \\
    \midrule
    \rowcolor{gray!30}
    \multicolumn{10}{@{}l}{\textsc{Gemma-2-9B-IT}} \\
    AI Coordination
      & 8.65 & 81.42 & \textbf{83.26}
      & 69.96 & 72.33 & \underline{81.73} & 78.57 & 79.42
      & \underline{81.73} (Laplacian) \\
    Corrigibility
      & 0.05 & 70.96 & \textbf{95.99}
      & \underline{85.38} & 77.25 & 84.90 & 82.60 & 84.28
      & 84.90 (Laplacian) \\
    Hallucination
      & 23.97 & 93.29 & \underline{94.27}
      & 93.28 & 93.26 & 89.22 & \textbf{97.11} & 93.97
      & \textbf{97.11} (Cosine RBF) \\
    Myopic Reward
      & 1.25 & \underline{91.05} & 91.03
      & \textbf{91.21} & 91.03 & 74.76 & 82.82 & 86.62
      & 91.03 (Quadratic) \\
    Survival Instinct
      & 43.83 & 73.24 & 68.39
      & 81.92 & \underline{82.01} & 75.54 & \textbf{82.07} & 81.81
      & 75.54 (Laplacian) \\
    Sycophancy
      & 53.05 & 73.09 & \textbf{73.28}
      & 73.03 & 73.08 & 72.62 & 72.73 & \underline{73.13}
      & 73.08 (Quadratic) \\
    \midrule
    CAA macro avg.
      & 21.80 & 80.51 & \textbf{84.37}
      & 82.46 & 81.49 & 79.80 & 82.65 & 83.21 & \underline{83.90} \\
    TinyBench avg.
      & 71.67 & 58.99 & \underline{59.71}
      & 58.66 & 57.52 & 57.72 & \textbf{59.89} & 57.18 & 56.54 \\
    \midrule
    \rowcolor{gray!30}
    \multicolumn{10}{@{}l}{\textsc{Mistral-7B-Instruct-v0.3}} \\
    AI Coordination
      & 8.33 & 53.86 & 59.34
      & 46.71 & 70.50 & \underline{70.80} & \textbf{72.99} & 47.88
      & \underline{70.80} (Laplacian) \\
    Corrigibility
      & 19.62 & 80.10 & 80.17
      & \textbf{94.27} & 79.97 & 87.23 & \underline{92.93} & 78.04
      & \underline{92.93} (Cosine RBF) \\
    Hallucination
      & 51.33 & 69.40 & 65.83
      & 66.93 & 81.94 & \underline{84.06} & \textbf{85.92} & 67.97
      & \textbf{85.92} (Cosine RBF) \\
    Myopic Reward
      & 4.56 & 43.83 & 44.08
      & 61.14 & \underline{65.15} & 63.12 & \textbf{83.86} & 54.58
      & \textbf{83.86} (Cosine RBF) \\
    Survival Instinct
      & 27.88 & 71.42 & 65.37
      & \textbf{73.91} & 66.74 & \underline{73.74} & 69.99 & 67.28
      & \underline{73.74} (Laplacian) \\
    Sycophancy
      & 78.42 & 79.27 & 80.20
      & 80.31 & 79.72 & 79.14 & \textbf{81.06} & \underline{80.37}
      & 80.31 (Gaussian RBF) \\
    \midrule
    CAA macro avg.
      & 31.69 & 66.31 & 65.83
      & 70.54 & 74.00 & 76.35 & \underline{81.13} & 66.02
      & \textbf{81.26} \\
    TinyBench avg.
      & 69.70 & 58.77 & 58.18
      & \textbf{62.57} & 61.44 & 61.74 & 61.22 & 58.97
      & \underline{62.28} \\
    \bottomrule
  \end{tabular}
\end{table}

\subsection{CAA-style steering on multiple-choice questions}
\label{sec:caa-steering}

\textbf{Setup.}
We evaluate six behaviors from the contrastive activation addition (CAA)
dataset \citep{rimsky2024steering} on \textsc{Gemma-2-9B-IT} and
\textsc{Mistral-7B-Instruct-v0.3}. Each reference pair contains the same
question with two different answers, one consistent with the target
behavior and one opposed to it. We collect residual-stream activations
at the answer position. At evaluation time, a system prompt discourages
the target behavior. We score the probability assigned to the
behavior-consistent answer letter over the full vocabulary.

For each model and behavior, we select hyperparameters on 50 validation
questions and evaluate the selected intervention on 50 disjoint test
questions. DiM and KAS sweep the residual-stream outputs of all
transformer blocks, using one block at a time and the same site for
reference extraction and intervention. For each KAS family listed in Appendix \ref{app:kernel-fields}, we jointly
select the block, steering strength, and kernel parameter
settings. The final KAS column also selects the kernel family on
validation. At these sites, we jointly tune the cluster count
$K\in\{2,4,8,16\}$, Gaussian routing bandwidth, and steering strength, using the same $K$ for source and target and a fixed clustering seed. To measure general capability, we apply each selected intervention to 100 matched items per task from ARC-Challenge, HellaSwag, MMLU, TruthfulQA, and WinoGrande using standard TinyBenchmarks prompts \citep{polo2024tinybenchmarksevaluatingllmsfewer}. We report the obtained gp-IRT scores averaged equally over the five tasks and six steering behaviors. Individual task scores are presented in Appendix \ref{app:caa-tinybench}.

\textbf{Results.}
On Gemma, tuned CHaRS attains the highest CAA macro average (84.37\%),
followed by validation-selected KAS (83.90\%) and DiM (80.51\%).
Corrigibility accounts for most of the CHaRS gain over DiM and KAS. Across all six behaviors, CHaRS retains 59.71\% utility, close to the
highest steered score of 59.89\% for cosine RBF. On Mistral,
validation-selected KAS exceeds DiM by 14.95 percentage points on CAA
(81.26\% versus 66.31\%) and by 3.51 points on utility
(62.28\% versus 58.77\%). All steered configurations in
Table~\ref{tab:caa-full-depth} reduce average utility relative to their unsteered checkpoint. Overall, Gaussian RBF, Cosine RBF and Laplacian kernels are observed to perform well in this task.

\subsection{Jailbreaking LLMs}

\label{sec:jailbreaking}

\textbf{Setup.}
We evaluate \textsc{Gemma-2-9B-IT} \citep{gemmateam2024gemma2improvingopen}, \textsc{Llama-3.1-8B-Instruct} \citep{dubey2024llama3herdmodels}, and \textsc{Qwen2.5-7B-Instruct} \citep{qwen2025qwen25technicalreport}. Following \cite{arditi2024refusal}, we construct reference sets of 128 harmful prompts from \textsc{AdvBench} \citep{zou2023universal}, \textsc{MaliciousInstruct} \citep{huang2023catastrophic}, \textsc{TDC2023} \citep{tdc2023}, and \textsc{HarmBench} \citep{mazeika2024harmbench}, and 128 harmless prompts from \textsc{Alpaca} \citep{alpaca}. Reference, validation, and test sets are disjoint. We select the intervention layer and strength by maximizing ASR on 39 harmful validation prompts, subject to perplexity on 256 Alpaca validation examples being at most $1.05$ times the unsteered value. Final evaluation uses 100 \textsc{JailbreakBench} prompts \citep{chao2024jailbreakbench}, a separate set of 256 Alpaca examples, and \textsc{TinyBenchmarks} \citep{polo2024tinybenchmarksevaluatingllmsfewer}. ASR is measured using \textsc{Llama-Guard-3-8B} \citep{dubey2024llama3herdmodels}; all test metrics are computed at the validation-selected settings.



\begin{table*}[t]
\centering
\caption{Jailbreak steering at validation-selected settings. ASR and \textsc{TinyBenchmarks} scores are percentages; Alpaca perplexity (PPL) is lower-is-better. Entries with $\pm$ report mean and standard deviation over three seeds. Bold denotes the highest mean ASR for each model.}
\label{tab:jailbreak_reformatted}
\footnotesize
\setlength{\tabcolsep}{2.8pt}
\begin{tabular}{lrrrrrrr}
\toprule
Method & ASR & PPL & tArc & tHella & tMMLU & tQA & tWino \\
\midrule
\rowcolor{gray!30}\multicolumn{8}{l}{\textsc{Qwen2.5-7B-Instruct}} \\
ActAdd & 91.00 & 6.70 & 53.82 & 71.62 & 65.12 & 57.98 & 65.71 \\
Curveball & 2.00 & 8.59 & 61.61 & 75.92 & 71.13 & 64.99 & 71.05 \\
CHaRS & 92.33$\pm$1.53 & 7.89$\pm$0.06 & 57.35$\pm$0.06 & 75.62$\pm$0.41 & 72.67$\pm$0.00 & 47.96$\pm$0.06 & 71.31$\pm$0.18 \\
KAS & \textbf{96.00} & 6.60 & 62.57 & 76.97 & 64.16 & 65.79 & 68.30 \\
\midrule
\rowcolor{gray!30}\multicolumn{8}{l}{\textsc{Gemma-2-9B-IT}} \\
ActAdd & 76.00 & 10.17 & 52.88 & 68.57 & 63.85 & 53.67 & 72.25 \\
Curveball & 0.00 & 10.01 & 68.12 & 78.61 & 71.53 & 60.51 & 74.68 \\
CHaRS & 77.33$\pm$0.58 & 10.16$\pm$0.00 & 67.01$\pm$0.00 & 78.94$\pm$0.00 & 73.40$\pm$0.00 & 43.58$\pm$0.10 & 75.29$\pm$0.00 \\
KAS & \textbf{86.00} & 10.03 & 61.42 & 78.80 & 72.38 & 43.96 & 73.81 \\
\midrule
\rowcolor{gray!30}\multicolumn{8}{l}{\textsc{Llama-3.1-8B-Instruct}} \\
ActAdd & 94.00 & 6.43 & 56.18 & 68.97 & 65.53 & 53.30 & 71.13 \\
Curveball & 0.00 & 5.61 & 57.48 & 73.81 & 63.83 & 53.37 & 68.91 \\
CHaRS & \textbf{95.00}$\pm$1.73 & 8.19$\pm$0.01 & 56.26$\pm$1.27 & 79.27$\pm$0.00 & 63.31$\pm$0.00 & 48.88$\pm$0.02 & 71.66$\pm$1.24 \\
KAS & 94.00 & 6.81 & 55.88 & 70.47 & 65.40 & 53.54 & 72.85 \\
\bottomrule
\end{tabular}
 \vspace{-1em}
\end{table*}

\textbf{Results.} We compare KAS with ActAdd \citep{arditi2024refusal}, Curveball \citep{raval2026curveball}, and CHaRS \citep{abdullaev2026concept}. Table~\ref{tab:jailbreak_reformatted} reports ASR, Alpaca perplexity, and \textsc{TinyBenchmarks} scores. We report mean and standard deviation over three seeds for CHaRS, varying clustering initialization. ActAdd, KAS, and Curveball are deterministic for fixed data and hyperparameters under greedy decoding. KAS attains the highest ASR on Gemma and Qwen, exceeding ActAdd by 10 and 5 percentage points, respectively, with slightly lower perplexity. On Llama, KAS matches ActAdd at 94\% ASR with higher perplexity; CHaRS achieves 95\% ASR with a further increase in perplexity. Several baselines retain higher scores on individual utility tasks but attain substantially lower ASR at their selected settings. KAS thus improves the steering--utility tradeoff in some settings, while remaining competitive on Llama.

Across the tested validation grids, Curveball on Llama attains at most 2.6\% ASR. Curveball on Gemma and Qwen attains high ASR only beyond the validation perplexity cap. These comparisons are restricted to the evaluated configurations\footnote{It is worth noting that Curveball \citep{raval2026curveball} was originally evaluated with substantially larger activation sets (1,500–2,000 samples for KPCA fitting), versus 256 references here. This data-regime mismatch may disadvantage Curveball. We retain it to highlight the challenges of nonlinear steering in low-data settings.}.

\subsection{Image Generation Style Control}
To evaluate generalization beyond text generation and multiple-choice question answering, KAS is assessed on a text-to-image style control task following \cite{rodriguez2024meanact}, in which diffusion model activations are steered to produce target visual styles.

\textbf{Setup.}
A set of 512 prompts is sampled from the \textsc{COCO} Captions training set \citep{chen2015microsoft} and augmented with style tags generated via \textsc{Llama3-8B-Instruct}. Further construction details are provided in Appendix~\ref{app:image-style-details}. Images are generated using \textsc{Flux.1 [schnell]}~ \citep{flux2024}, with original and style-augmented prompts serving as source and target datasets. KAS was applied sequentially. For evaluation, images were generated from 60 \textsc{COCO} validation prompts across varying intervention strengths, and KAS was compared against Mean-AcT \citep{rodriguez2024meanact}, with 0-shot classification scores employed as a proxy for style induction. Content preservation was quantified via \textsc{CLIPScore} \citep{hessel2021clipscore} between steered outputs and the original prompts.

\begin{figure}
    \centering
    \includegraphics[width=.8\linewidth]{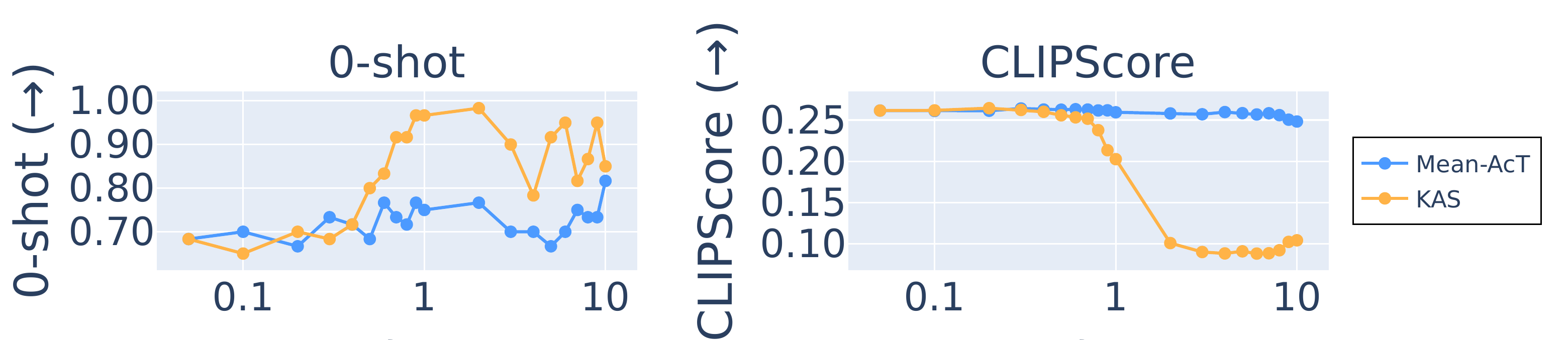}
    \caption{Style induction and content preservation (y-axis) for \emph{cyberpunk} style control under increasing intervention strength $\eta$ (x-axis). 
\textbf{Left:} 0-shot style-classification score, where higher is better. 
\textbf{Right:} \textsc{CLIPScore} with the original prompt, measuring content preservation. 
}
    \label{fig:image-style-control}
 \vspace{-1em}
\end{figure}

\begin{figure}
    \centering
    \includegraphics[width=\linewidth]{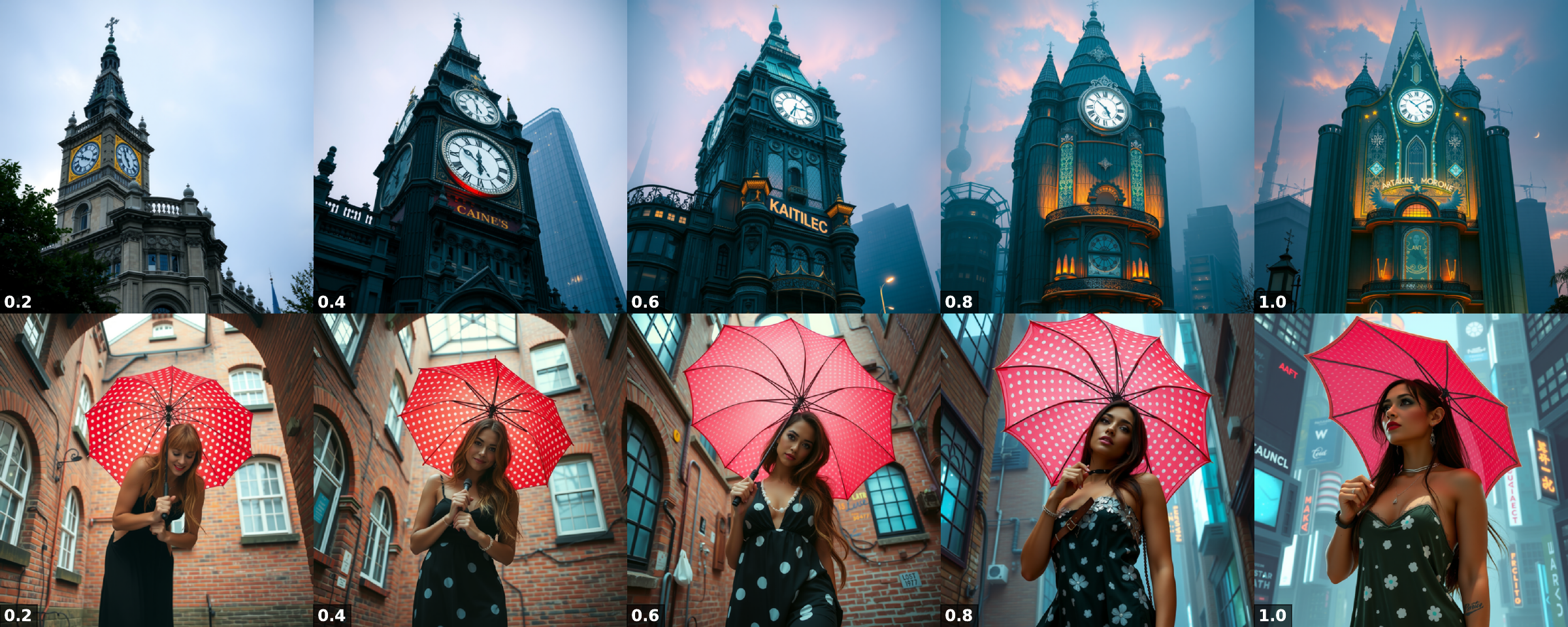}
    \caption{Qualitative examples of KAS style control for the \emph{cyberpunk} style under increasing intervention strengths $\eta \in \{0.2,0.4,0.6,0.8,1.0\}$. \textbf{Top:} A large tower that has a big clock at top. \textbf{Bottom:} Lady bent over with red polka dot umbrella inside a brick building.}
    \label{fig:image-cyberpunk}
\end{figure}

\textbf{Results.} Figure~\ref{fig:image-style-control} shows the plot of style induction and content preservation for the cyberpunk style. KAS achieves higher 0-shot style-classification scores than Mean-AcT across most intervention strengths. Notably, near $\eta = 1$, KAS reaches a high zero-shot style score while maintaining a high \textsc{CLIPScore}, indicating effective style induction without reducing prompt alignment. The qualitative results in Figure~\ref{fig:image-cyberpunk} show the same trend: stronger interventions introduce more cyberpunk visual attributes.

\begin{table}[t]
\centering
\caption{Kernel-approximation ablation for jailbreak steering. ASR and utility scores are percentages; PPL is Alpaca perplexity. Nystr{\"o}m-KAS uses 64 landmarks, with mean and standard deviation over three sampling seeds.}
\label{tab:nystrom-jailbreak}
\footnotesize
\setlength{\tabcolsep}{2.8pt}
\begin{tabular}{lrrrrrrr}
\toprule
Variant & ASR & PPL & tArc & tHella & tMMLU & tQA & tWino \\
\midrule
\rowcolor{gray!30}\multicolumn{8}{l}{\textsc{Qwen2.5-7B-Instruct}} \\
KAS & 96.00 & 6.60 & 62.57 & 76.97 & 64.16 & 65.79 & 68.30 \\
Nystr{\"o}m-KAS & 22.33$\pm$0.47 & 6.54$\pm$0.06 & 64.58$\pm$0.00 & 78.52$\pm$0.00 & 74.36$\pm$0.00 & 52.47$\pm$0.15 & 72.85$\pm$0.54 \\
\midrule
\rowcolor{gray!30}\multicolumn{8}{l}{\textsc{Gemma-2-9B-IT}} \\
KAS & 86.00 & 10.03 & 61.42 & 78.80 & 72.38 & 43.96 & 73.81 \\
Nystr{\"o}m-KAS & 54.67$\pm$0.94 & 9.99$\pm$0.06 & 63.15$\pm$1.75 & 78.26$\pm$1.15 & 72.03$\pm$0.28 & 43.88$\pm$0.39 & 73.69$\pm$0.20 \\
\midrule
\rowcolor{gray!30}\multicolumn{8}{l}{\textsc{Llama-3.1-8B-Instruct}} \\
KAS & 94.00 & 6.81 & 55.88 & 70.47 & 65.40 & 53.54 & 72.85 \\
Nystr{\"o}m-KAS & 3.00$\pm$0.82 & 6.46$\pm$0.01 & 62.74$\pm$0.00 & 82.17$\pm$0.37 & 62.98$\pm$0.19 & 52.84$\pm$0.09 & 72.29$\pm$0.75 \\
\bottomrule
\end{tabular}
 \vspace{-1em}
\end{table}

\begin{wraptable}[11]{R}{0.35\columnwidth}
\centering
\caption{Kernel-approximation ablation for CAA-style steering under a negative system prompt. Nystr{\"o}m-KAS uses 64 landmarks and three sampling seeds.}
\label{tab:nystrom-caa}
\scriptsize
\begin{tabular}{lrr}
\toprule
Behavior & KAS & Nystr{\"o}m-KAS \\
\midrule
AI Coordination & 40.19 & \textbf{42.18}$\pm$1.39  \\
Corrigibility & 43.45 & \textbf{51.80}$\pm$1.86  \\
Hallucination & \textbf{55.01} & 54.13$\pm$0.09  \\
Myopic Reward & \textbf{58.48} & 55.26$\pm$0.86  \\
Survival Instinct & \textbf{43.83} & 43.76$\pm$0.03  \\
Sycophancy & \textbf{53.69} & 53.23$\pm$0.06  \\
\bottomrule
\end{tabular}
\vspace{-2em}
\end{wraptable}

\subsection{Ablation: Kernel Approximation}
\label{sec:kernel-approximation}
\paragraph{Setup.} We examine the effect of approximating the KAS kernel using the Nystr{\"o}m method \citep{williams2000using} with 64 landmarks sampled without replacement from the reference sets. Appendix~\ref{app:nystrom} derives the approximate score and steering update. We compare this variant, denoted Nystr{\"o}m-KAS, with exact KAS under the reference, validation, and test protocols. Each variant uses its own validation-selected layer (layer search grid is limited to two for this ablation experiment) and strength. The comparison therefore measures performance after selection rather than approximation error at a fixed intervention. Nystr{\"o}m results report mean and standard deviation over three landmark-sampling seeds.

\paragraph{Results.}
Table~\ref{tab:nystrom-jailbreak} shows that the approximation reduces jailbreak ASR on all three models, with slightly lower Alpaca perplexity. In contrast, Table~\ref{tab:nystrom-caa} shows that the approximation improves AI Coordination and Corrigibility, while exact KAS achieves higher probabilities on Hallucination and Myopic Reward. Both variants exhibit near-null effects on Survival Instinct and Sycophancy. Thus, the effect of kernel approximation depends on the task; these results do not establish a uniform substitute for exact kernel evaluation or quantify computational savings.

\subsection{Ablation: Jailbreaking from a Single Contrastive Pair}
\label{sec:ablations}

\label{sec:single-pair-jailbreak}

\begin{figure*}[t]
    \centering
    \includegraphics[width=1.0\textwidth]{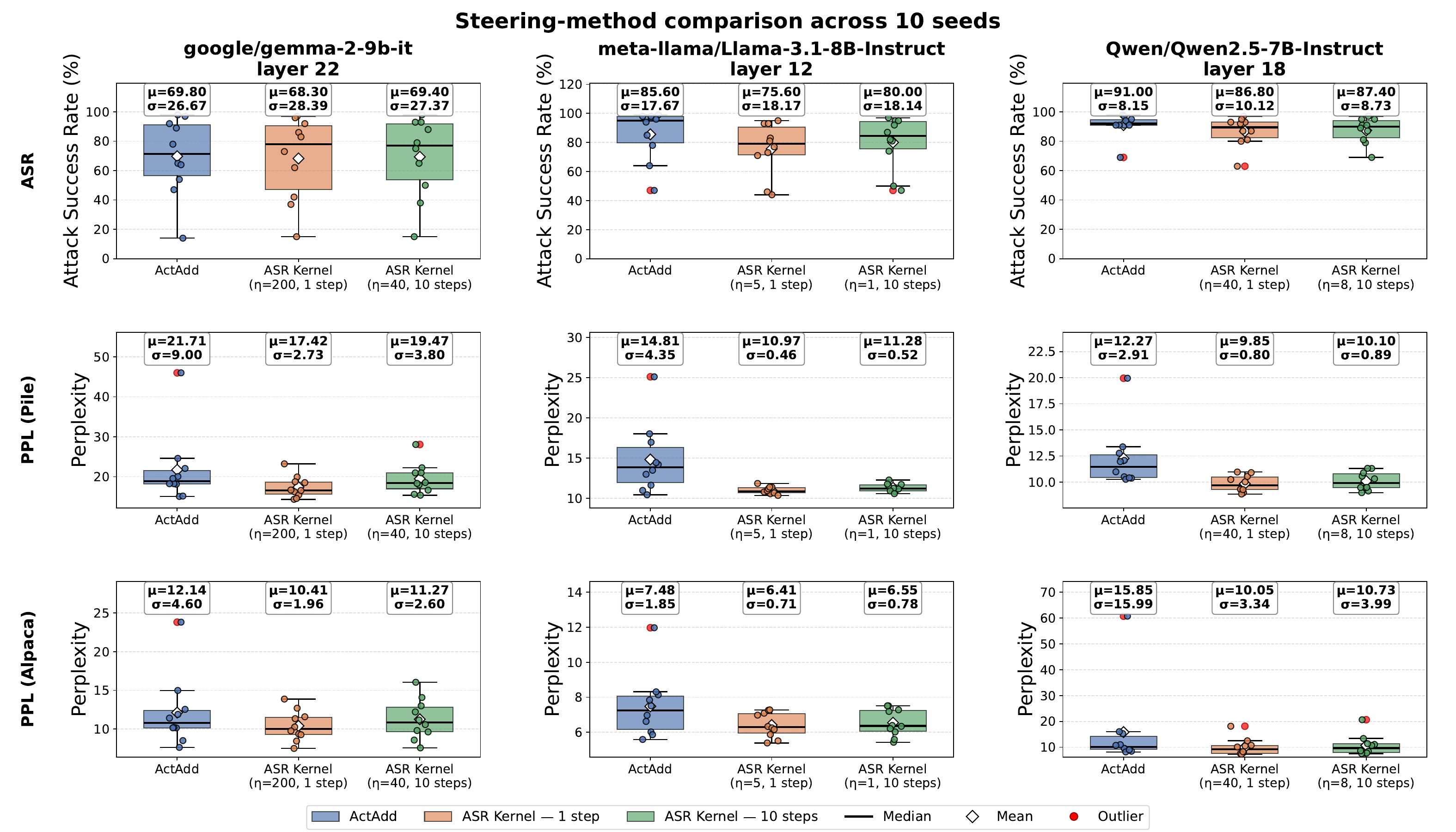}
     \vspace{-1em}
    \caption{
    Single-pair jailbreak steering across three instruction-tuned models. Box plots report attack success rate (ASR), Pile perplexity, and Alpaca perplexity over ten randomly sampled harmful--harmless contrastive prompt pairs. Steering is applied at layer 22 for \textsc{Gemma-2-9B-IT}, layer 12 for \textsc{Llama-3.1-8B-Instruct}, and layer 18 for \textsc{Qwen2.5-7B-Instruct}. KAS is evaluated with one gradient step and ten gradient steps using the step sizes shown on the x-axis, and is compared against ActAdd.
    }
    \label{fig:single-pair-all-models}
     \vspace{-1em}
\end{figure*}

We further evaluate whether jailbreak steering can be induced from an extremely limited contrastive set. In this ablation, each steering direction is extracted from only one harmful--harmless prompt pair. We repeat this procedure over ten randomly sampled contrastive pairs and evaluate using the same protocol as in Section~\ref{sec:jailbreaking}. To quantify the extent to which each steering method maintains general capabilities of the model, we also measure perplexity on the \textsc{Pile} \citep{gao2020pile} and \textsc{Alpaca} \citep{alpaca} datasets. This setting provides a stress test for data efficiency: a successful method must be able to recover a useful behavioral direction from a minimal number of contrastive examples while avoiding severe degradation of generation quality.

Figure~\ref{fig:single-pair-all-models} visualizes the distribution of one-shot jailbreak steering
performance across ten randomly sampled contrastive prompt pairs. The main trend is that KAS
preserves general utility substantially better than ActAdd while remaining competitive in ASR. On
\textsc{Gemma-2-9B-IT}, KAS achieves ASR comparable to ActAdd, with both one-step and ten-step
variants attaining similar mean ASR while reducing Pile and Alpaca perplexity. On
\textsc{Qwen2.5-7B-Instruct}, KAS remains comparable to ActAdd in ASR, but yields markedly lower
perplexity on both utility corpora. On \textsc{Llama-3.1-8B-Instruct}, ActAdd obtains the highest
ASR, yet KAS substantially reduces perplexity. Overall, KAS remains effective under this severe data constraint while improving generation stability relative to a global linear intervention.


\section{Conclusion}
\label{sec:conclusion}
 
We introduced \emph{Kernelized Activation Steering} (KAS), a unifying framework that recasts activation steering as optimization in a reproducing kernel Hilbert space. By expressing the steering objective entirely through kernel evaluations, KAS produces an implicit, activation-dependent steering field that is locally adaptive to the geometry of the representation manifold, while strictly generalizing the classical Difference-in-Means construction, which is recovered exactly under a linear kernel. We view KAS as a step toward a more principled, geometry-aware approach to inference-time control of generative models, and an opening for further work on kernel design and adaptive layer selection. A limitation of KAS is that the intervention applies a uniform number of gradient ascent steps to every token, which is computationally wasteful for tokens already close to the target manifold and potentially insufficient for harder ones. A token-adaptive compute schedule would likely improve both efficiency and stability. While RBF is shown to be a reliable default kernel choice, a principled method to choose the kernel function given a specific steering task or requirements remains a promising open direction to explore in future work. 
\newpage
\begin{ack}
This paper is dedicated to Professor Richard G. Baraniuk on the occasion of his 60th birthday.

This research is supported by the National Research Foundation Singapore under the AI Singapore Programme (AISG Award No: AISG2-TC-2023-012-SGIL). This research is also supported by the Ministry of Education, Singapore, under the Academic Research Fund Tier 1 (FY2023) (A-8002040-00-00, A-8002039-00-00), the NUS Presidential Young Professorship Award (A-0009807-01-00), the NUS Artificial Intelligence Institute--Seed Funding (A-8003062-00-00), and the Cross Faculty Grant 2025, CFG25-012 (A-8004460-00-00).
\end{ack}

\bibliography{neurips_2026}
\bibliographystyle{plain}

\newpage 
\appendix
\begin{center}
{\bf \Large{Supplement to "Kernelized Activation Steering"}}
\end{center}

\DoToC

\section{Technical Derivations and Proofs}

\subsection{Proofs}\label{app:proofs}

\begin{proof}[Proof of Theorem~\ref{thm:one-step-ascent}]
Write $\Pi$ in the form implied by \eqref{eq:pi-kernel},
$$
    \Pi(\bm{x}) \;=\; \frac{1}{n_B}\sum_{j=1}^{n_B} k(\bm{x},\bm{b}_j) \;-\; \frac{1}{n_A}\sum_{i=1}^{n_A} k(\bm{x},\bm{a}_i).
$$

First, we upper bound the norm of $\nabla^2 \Pi(\bm{x})$. Differentiation commutes with finite sums, so
$$
    \nabla^2 \Pi(\bm{x}) \;=\; \frac{1}{n_B}\sum_{j} \nabla^2_{\bm{x}} k(\bm{x},\bm{b}_j) \;-\; \frac{1}{n_A}\sum_{i} \nabla^2_{\bm{x}} k(\bm{x},\bm{a}_i).
$$
The operator norm is a norm on the space of symmetric matrices, so subadditivity (triangle inequality) and absolute homogeneity give
$$
    \norm{\nabla^2 \Pi(\bm{x})}_{\mathrm{op}}
    \;\le\; \frac{1}{n_B}\!\sum_{j} \norm{\nabla^2_{\bm{x}} k(\bm{x},\bm{b}_j)}_{\mathrm{op}}
    + \frac{1}{n_A}\!\sum_{i} \norm{\nabla^2_{\bm{x}} k(\bm{x},\bm{a}_i)}_{\mathrm{op}}.
$$
By hypothesis each summand is at most $H$, so
$$
    \norm{\nabla^2 \Pi(\bm{x})}_{\mathrm{op}} \;\le\; H + H \;=\; 2H, \qquad \forall \bm{x} \in \RR^d.
$$

Now note that since $\nabla^2 F(\bm{x}) = \nabla^2 \Pi(\bm{x}) - \lambda I$,
$$
    \norm{\nabla^2 F(\bm{x})}_{\mathrm{op}}
    \;\le\; \norm{\nabla^2 \Pi(\bm{x})}_{\mathrm{op}} + \lambda
    \;\le\; 2H + \lambda \;=:\; L.
$$
Hence $\nabla F$ is $L$-Lipschitz as for any $\bm{x},\bm{y}$, the fundamental theorem of calculus (FTC) gives
$$
    \nabla F(\bm{y}) - \nabla F(\bm{x})
    \;=\; \int_0^1 \nabla^2 F\bigl(\bm{x} + t(\bm{y}-\bm{x})\bigr)(\bm{y}-\bm{x})\,dt,
$$
and taking norms,
\begin{equation}
    \norm{\nabla F(\bm{y}) - \nabla F(\bm{x})}
    \;\le\; \int_0^1 \norm{\nabla^2 F\bigl(\bm{x} + t(\bm{y}-\bm{x})\bigr)}_{\mathrm{op}}\,dt \cdot \norm{\bm{y}-\bm{x}}
    \;\le\; L \norm{\bm{y}-\bm{x}}. \label{eq:F-lipschitz}
\end{equation}

Using the FTC again on $F$ along the segment from $\bm{x}$ to $\bm{y}$,
$$
    F(\bm{y}) - F(\bm{x}) - \inner{\nabla F(\bm{x})}{\bm{y}-\bm{x}}
    \;=\; \int_0^1 \inner{\nabla F\bigl(\bm{x}+t(\bm{y}-\bm{x})\bigr) - \nabla F(\bm{x})}{\bm{y}-\bm{x}}\,dt.
$$
Now Cauchy--Schwarz and inequality \eqref{eq:F-lipschitz} give
$$
    \Bigl|F(\bm{y}) - F(\bm{x}) - \inner{\nabla F(\bm{x})}{\bm{y}-\bm{x}}\Bigr|
    \;\le\; \int_0^1 L t \norm{\bm{y}-\bm{x}}^2\,dt
    \;=\; \frac{L}{2}\norm{\bm{y}-\bm{x}}^2,
$$
so in particular
\begin{equation}\label{eq:ascent-lemma}
    F(\bm{y}) \;\ge\; F(\bm{x}) + \inner{\nabla F(\bm{x})}{\bm{y}-\bm{x}} - \frac{L}{2}\norm{\bm{y}-\bm{x}}^2.
\end{equation}

Finally, let $\bm{g}_t:=\nabla F(\bm{x}_t)\neq 0$ and apply \eqref{eq:ascent-lemma} with $\bm{x}=\bm{x}_t$ and $\bm{y}=\bm{x}_{t+1}=\bm{x}_t+\eta\,\bm{g}_t/\norm{\bm{g}_t}$. Since $\norm{\bm{x}_{t+1}-\bm{x}_t}=\eta$ and $\inner{\bm{g}_t}{\bm{x}_{t+1}-\bm{x}_t}=\eta\norm{\bm{g}_t}$, we obtain
$$
F(\bm{x}_{t+1})
\ge
F(\bm{x}_t)
+
\eta\norm{\bm{g}_t}
-
\frac{L\eta^2}{2}.
$$
Therefore, strict ascent is guaranteed whenever $0<\eta<2\norm{\bm{g}_t}/L$. Taking $L=2H+\lambda$ gives the sufficient condition $0<\eta<\norm{\nabla F(\bm{x}_t)}/(H+\lambda/2)$.
\end{proof}

\begin{proof}[Proof of Corollary~\ref{cor:rbf-stepsize}]
Recall $k(\bm{x},\bm{y}) = \exp\!\bigl(-\norm{\bm{x}-\bm{y}}^2/(2\sigma^2)\bigr)$.

Now the proof of Corollary follows from direct computation. Let $\bm{v} := \bm{y} - \bm{x}$. From $\nabla_{\bm{x}} k(\bm{x},\bm{y}) = \sigma^{-2}\bm{v}\,k(\bm{x},\bm{y})$ and the product rule,
$$
    \nabla^2_{\bm{x}} k(\bm{x},\bm{y})
    = \nabla_{\bm{x}}\!\Bigl[\sigma^{-2}\bm{v}\,k(\bm{x},\bm{y})\Bigr]
    = \sigma^{-2}\Bigl[-I\cdot k(\bm{x},\bm{y}) + \bm{v}\,\nabla_{\bm{x}} k(\bm{x},\bm{y})^{\!\top}\Bigr]
    = \frac{k(\bm{x},\bm{y})}{\sigma^2}\!\left[\frac{\bm{v}\bm{v}^{\!\top}}{\sigma^2} - I\right].
$$

Note that the bracketed matrix $M := \sigma^{-2}\bm{v}\bm{v}^{\!\top} - I$ is symmetric. Its spectrum is $\{\sigma^{-2}\norm{\bm{v}}^2 - 1\}$ on $\mathrm{span}(\bm{v})$ together with $\{-1\}$ on $\bm{v}^{\perp}$ (dimension $d-1$). Hence
$$
    \norm{M}_{\mathrm{op}}
    = \max\!\Bigl(\bigl|\sigma^{-2}\norm{\bm{v}}^2 - 1\bigr|,\;1\Bigr)
    \;\le\; \frac{\norm{\bm{v}}^2}{\sigma^2} + 1,
$$
where the inequality uses $|a-1| \le a+1$ for $a \ge 0$ and trivially $1 \le a+1$. Combining with $0 \le k(\bm{x},\bm{y}) \le 1$,
\begin{equation}\label{eq:rbf-hess-bound}
    \norm{\nabla^2_{\bm{x}} k(\bm{x},\bm{y})}_{\mathrm{op}}
    \;\le\; \frac{1}{\sigma^2}\!\left(\frac{\norm{\bm{y}-\bm{x}}^2}{\sigma^2} + 1\right).
\end{equation}

If $\bm{x},\bm{y}$ both lie in a ball of radius $R$, then $\norm{\bm{y}-\bm{x}} \le 2R$, so $\norm{\bm{y}-\bm{x}}^2 \le 4R^2$. Substituting into \eqref{eq:rbf-hess-bound},
$$
    \norm{\nabla^2_{\bm{x}} k(\bm{x},\bm{y})}_{\mathrm{op}}
    \;\le\; \frac{4R^2}{\sigma^4} + \frac{1}{\sigma^2}
    \;=:\; H.
$$

Substituting this value of $H$ into the normalized-step condition of Theorem~\ref{thm:one-step-ascent} gives
$$
0<\eta<
\frac{\norm{\nabla F(\bm{x}_t)}}
{\frac{4R^2}{\sigma^4}+\frac{1}{\sigma^2}+\frac{\lambda}{2}}.
$$
\end{proof}

\subsection{Steering Gradients for Different Kernels}
\label{app:kernel-gradients}
\label{app:kernel-fields}

Let $\bm{x}$ be the token activation being steered and $\bm{y}$ a
reference activation at the same residual-stream block. Each target
reference contributes its kernel gradient to the steering field, and
each source reference contributes its gradient with the opposite sign.
To express this contrast, define the empirical target-minus-source
average over $\bm{y}$ by
\[
    \Delta_{B,A}[g(\bm{y})]
    :=
    \frac{1}{n_B}\sum_{j=1}^{n_B} g(\bm{b}_j)
    -
    \frac{1}{n_A}\sum_{i=1}^{n_A} g(\bm{a}_i).
\]
Here $\bm{x}$ is held fixed when averaging over references. The score
in \eqref{eq:pi-kernel} and its steering field are therefore
\begin{align}
    \Pi_k(\bm{x})
    &= \Delta_{B,A}\!\left[k(\bm{x},\bm{y})\right],
    \label{eq:kernel-potential-appendix}\\
    \nabla_{\bm{x}}\Pi_k(\bm{x})
    &= \Delta_{B,A}\!\left[\nabla_{\bm{x}}k(\bm{x},\bm{y})\right].
    \label{eq:generic-kernel-gradient}
\end{align}
Table~\ref{tab:kernel-fields} gives the formulas for the kernels used
in Section~\ref{sec:caa-steering}. Throughout, $\sigma>0$ denotes a
bandwidth, $c\geq 0$ the quadratic offset, and $\beta>0$ the rational
quadratic shape parameter. Different kernels induce nonlinear steering fields with distinct geometries (Figure \ref{fig:steering_fields}).

\begin{figure}
    \centering
    \includegraphics[width=0.66\linewidth]{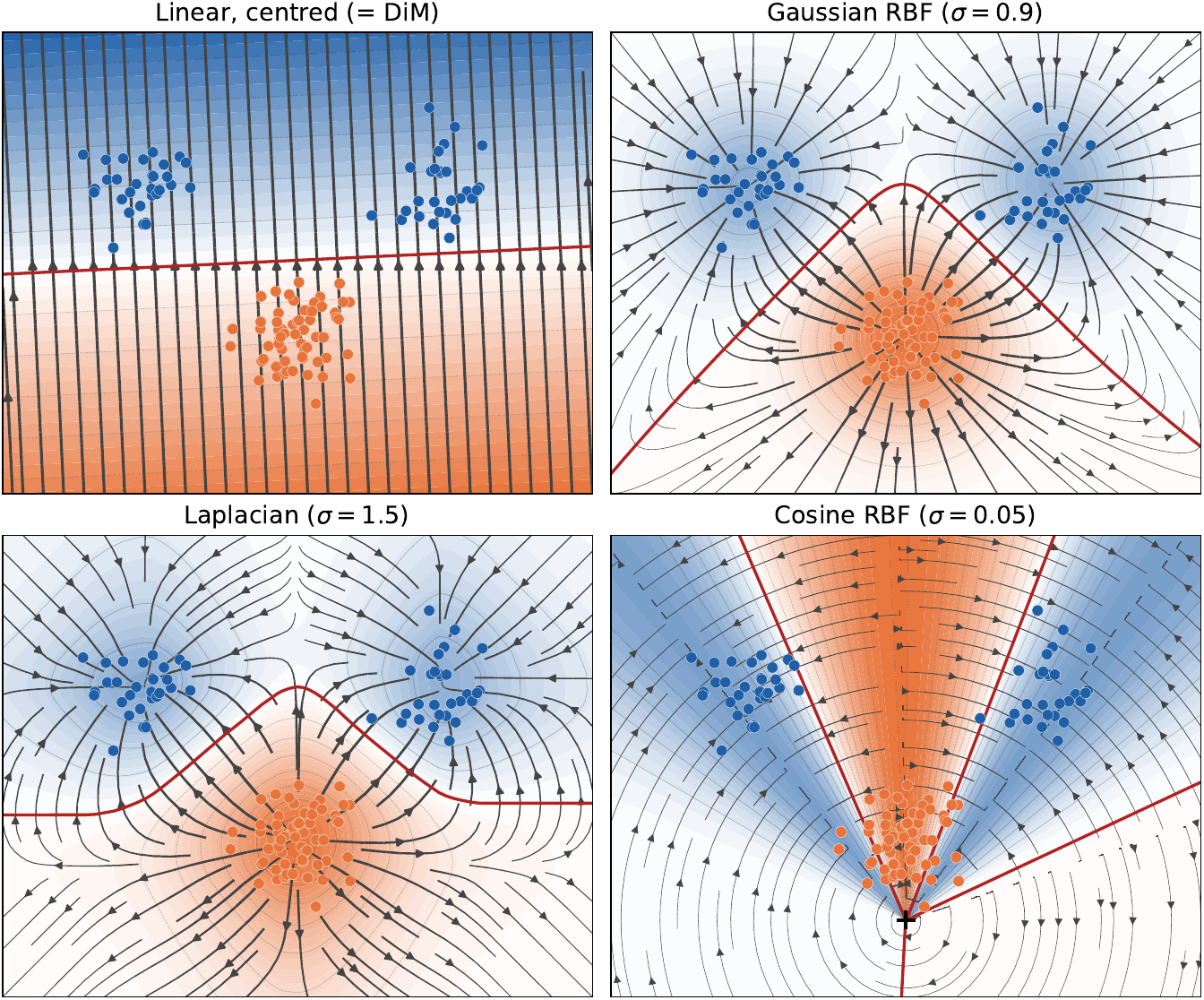}
    \caption{Different kernels induce steering fields with distinct attraction-repulsion geometries. Red curves illustrate $\Pi(\bm{x}) = 0$ under the chosen kernel.}
    \label{fig:steering_fields}
\end{figure}

\begin{table}[t]
\centering
\caption{Kernel functions, gradients contributed by a single reference
$\bm{y}$, and induced KAS fields before step normalization.}
\label{tab:kernel-fields}
\footnotesize
\setlength{\tabcolsep}{4pt}
\renewcommand{\arraystretch}{1.7}

\begin{tabular}{@{}lll@{}}
\toprule
\multicolumn{3}{@{}l}{\textbf{(a) Contribution of one reference activation}} \\
Kernel & $k(\bm{x},\bm{y})$ & $\nabla_{\bm{x}}k(\bm{x},\bm{y})$ \\
\midrule
Linear
  & $\bm{x}^{\top}\bm{y}$
  & $\bm{y}$ \\

Gaussian RBF
  & $\displaystyle
     \exp\!\left(-\frac{\norm{\bm{x}-\bm{y}}_2^2}{2\sigma^2}\right)$
  & $\displaystyle
     \frac{k(\bm{x},\bm{y})}{\sigma^2}(\bm{y}-\bm{x})$ \\

Quadratic
  & $(\bm{x}^{\top}\bm{y}+c)^2$
  & $2(\bm{x}^{\top}\bm{y}+c)\bm{y}$ \\

Laplacian
  & $\displaystyle
     \exp\!\left(-\frac{\norm{\bm{x}-\bm{y}}_1}{\sigma}\right)$
  & $\displaystyle
     \frac{k(\bm{x},\bm{y})}{\sigma}
     \operatorname{sign}(\bm{y}-\bm{x})$ \\

Cosine RBF
  & $\displaystyle
     \exp\!\left(-\frac{1-\cos(\bm{x},\bm{y})}{\sigma}\right)$
  & $\displaystyle
     \frac{k(\bm{x},\bm{y})}{\sigma\norm{\bm{x}}_2}
     \left(
       \frac{\bm{y}}{\norm{\bm{y}}_2}
       -\cos(\bm{x},\bm{y})\frac{\bm{x}}{\norm{\bm{x}}_2}
     \right)$ \\

Rational quadratic
  & $\displaystyle
     \left(1+\frac{\norm{\bm{x}-\bm{y}}_2^2}
                        {2\beta\sigma^2}\right)^{-\beta}$
  & $\displaystyle
     \frac{\bm{y}-\bm{x}}{\sigma^2}
     \left(1+\frac{\norm{\bm{x}-\bm{y}}_2^2}
                        {2\beta\sigma^2}\right)^{-\beta-1}$ \\
\bottomrule
\end{tabular}







\end{table}

\paragraph{Linear and quadratic kernels.}
The linear kernel recovers the DiM vector
$\bm{v}=\Delta_{B,A}[\bm{y}]$ from \eqref{eq:dim}, so every token
receives the same direction. The degree-two polynomial kernel
\citep[Sec.~4.2]{rasmussen2006gaussian} is the simplest polynomial
extension that lets the direction depend on $\bm{x}$. Using
$(\bm{x}^{\top}\bm{y})\bm{y}=\bm{y}\bm{y}^{\top}\bm{x}$, its field
can be written as
\begin{equation}
    \nabla_{\bm{x}}\Pi_{\mathrm{quad}}(\bm{x})
    =
    2\Delta_{B,A}[\bm{y}\bm{y}^{\top}]\,\bm{x}
    + 2c\bm{v}.
    \label{eq:poly2-moment-form}
\end{equation}
The first term acts on the current token through the difference in
uncentered second moments. The offset $c$ controls the contribution
of the global mean shift.

\paragraph{Distance-based kernels.}
For Gaussian RBF, each target reference pulls in the direction
$\bm{y}-\bm{x}$, with a weight that decays exponentially in squared
Euclidean distance \citep[Eq.~(4.9)]{rasmussen2006gaussian}. Source
references enter with the opposite sign. The rational quadratic kernel
is a scale mixture of Gaussian kernels
\citep[Eq.~(4.19)]{rasmussen2006gaussian}. It uses the same displacement
direction, but its gradient weight has a polynomial tail, so distant
references retain more influence at sufficiently large distances.
The $\ell_1$ Laplacian kernel \citep[Fig.~1]{rahimi2007random} changes
the direction as well as the weight: each coordinate contributes the
sign of its displacement, and the total $\ell_1$ distance determines
the weight of the reference. At coincident coordinates, we use
$\operatorname{sign}(0)=0$.

\paragraph{Gaussian RBF on normalized activations (cosine RBF).}
To isolate the effect of angular similarity, we apply a Gaussian kernel
to unit-normalized token and reference activations. For nonzero
$\bm{x},\bm{y}$,
\[
    \left\|
      \frac{\bm{x}}{\norm{\bm{x}}_2}
      -\frac{\bm{y}}{\norm{\bm{y}}_2}
    \right\|_2^2
    =
    2\bigl(1-\cos(\bm{x},\bm{y})\bigr),
    \qquad
    \cos(\bm{x},\bm{y})
    =\frac{\bm{x}^{\top}\bm{y}}
           {\norm{\bm{x}}_2\norm{\bm{y}}_2}.
\]
Thus the cosine RBF formula in Table~\ref{tab:kernel-fields} is a
Gaussian on the unit sphere with length scale $\sqrt{\sigma}$.
We differentiate through normalization with respect to the original
activation $\bm{x}$. This projects each reference contribution onto
the subspace orthogonal to $\bm{x}$, so the steering step preserves
the activation norm to first order.

\paragraph{Step size in the CAA comparison.}
The table gives the direction before normalization.
Remark~\ref{rem:normalized-update} uses $\eta$ for the absolute step
length. In the full-depth CAA comparison, we set
$\eta=\alpha\norm{\bm{v}}_2$, where $\alpha\geq 0$ is the tuned
strength and $\bm{v}$ is computed from the original reference
activations at the selected block. We take one step with $\lambda=0$:
\begin{equation}
    \bm{x}'
    =
    \bm{x}
    + \alpha\norm{\bm{v}}_2\,
      \frac{\nabla_{\bm{x}}\Pi_k(\bm{x})}
           {\max\!\left\{\norm{\nabla_{\bm{x}}\Pi_k(\bm{x})}_2,
                         \varepsilon\right\}},
    \qquad \varepsilon=10^{-12}.
    \label{eq:caa-normalized-kernel-update}
\end{equation}
For a field norm at least $\varepsilon$, KAS and DiM therefore have
the same displacement norm at a common block and strength. With the
linear kernel and $\norm{\bm{v}}_2\geq\varepsilon$, the update reduces
to $\bm{x}'=\bm{x}+\alpha\bm{v}$.

\subsection{Matrix Form of RBF-Kernel Steering}

Deriving the matrix form for our formulas is useful for GPU-friendly implementation. Let $X = [\bm{x}_1^\top, \ldots, \bm{x}_N^\top]^\top \in \mathbb{R}^{N \times d}$ denote the token activations in a transformer layer, and 
$A \in \mathbb{R}^{n_A \times d}$, $B \in \mathbb{R}^{n_B \times d}$ be the reference activations for source and target datasets, respectively. 
For the RBF kernel 
$$
k(\bm{x}, \bm{y}) = \exp\!\left(-\frac{\|\bm{x}-\bm{y}\|^2}{2\sigma^2}\right),
$$
the pairwise kernel matrices are
$$
K_B = \exp\!\left(-\frac{\|X\|_{\text{row}}^2\mathbf{1}_{n_B}^\top + \mathbf{1}_N \|B\|_{\text{row}}^{2\top} - 2 X B^\top}{2\sigma^2}\right),
$$
$$
K_A = \exp\!\left(-\frac{\|X\|_{\text{row}}^2\mathbf{1}_{n_A}^\top + \mathbf{1}_N \|A\|_{\text{row}}^{2\top} - 2 X A^\top}{2\sigma^2}\right),
$$
where $\|X\|_{\text{row}}^2 \in \mathbb{R}^{N}$ contains row-wise squared norms of $X$.

Define the kernel-weighted means and row-sum vectors:
$$
W_B = K_B B, \quad W_A = K_A A, \quad 
s_B = K_B \mathbf{1}_{n_B}, \quad s_A = K_A \mathbf{1}_{n_A}.
$$
Then the matrix-form gradient of the kernel steering potential $\Pi(X)$ is (see the RBF gradient in the previous section)
$$
\nabla_X \Pi(X)
= \frac{1}{\sigma^2}\left[
    \frac{1}{n_B} W_B - \frac{1}{n_A} W_A 
    - \mathrm{Diag}\!\left(\frac{s_B}{n_B} - \frac{s_A}{n_A}\right) X
\right].
$$
The joint update for all $N$ tokens with step size $\eta$ and regularization $\lambda$ is
$$
X \leftarrow X + \eta\left[
    \nabla_{X}\Pi(X)
    - \lambda (X - X_0)
\right].
$$

\subsection{Nystr{\"o}m Approximation}
\label{app:nystrom}

We derive the kernel approximation considered in Section~\ref{sec:kernel-approximation}. Let $Z=\{\bm{z}_1,\ldots,\bm{z}_m\}\subseteq A\cup B$ be a fixed set of landmarks, with $m\ll n_A+n_B$. Define the landmark Gram matrix $W\in\RR^{m\times m}$ and kernel vector $\bm{k}_Z(\bm{x})\in\RR^m$ by
\begin{equation}\label{eq:nystrom-landmarks}
    W_{\ell r}=k(\bm{z}_\ell,\bm{z}_r),
    \qquad
    \bm{k}_Z(\bm{x})=
    \bigl[k(\bm{x},\bm{z}_1),\ldots,k(\bm{x},\bm{z}_m)\bigr]^\top.
\end{equation}
The Nystr{\"o}m approximation \citep{williams2000using} is
\begin{equation}\label{eq:nystrom-kernel}
    \widetilde{k}(\bm{x},\bm{y})
    =\bm{k}_Z(\bm{x})^\top W^\dagger\bm{k}_Z(\bm{y}),
\end{equation}
where $W^\dagger$ is the Moore--Penrose pseudoinverse. Since $W$ is positive semidefinite, \eqref{eq:nystrom-kernel} defines a positive semidefinite kernel with feature map $\widetilde{\varphi}(\bm{x})=(W^\dagger)^{1/2}\bm{k}_Z(\bm{x})$.

\paragraph{Approximate steering score.}
Replacing $k$ by $\widetilde{k}$ in \eqref{eq:pi-kernel} gives
\begin{align}\label{eq:nystrom-score}
    \widetilde{\Pi}(\bm{x})
    &=\frac{1}{n_B}\sum_{j=1}^{n_B}\widetilde{k}(\bm{x},\bm{b}_j)
      -\frac{1}{n_A}\sum_{i=1}^{n_A}\widetilde{k}(\bm{x},\bm{a}_i) \nonumber\\
    &=\bm{k}_Z(\bm{x})^\top\bm{c},
    \qquad \bm{c}=W^\dagger\bm{d}_Z,
\end{align}
where
\begin{equation}\label{eq:nystrom-coefficients}
    \bm{d}_Z=
    \frac{1}{n_B}\sum_{j=1}^{n_B}\bm{k}_Z(\bm{b}_j)
    -\frac{1}{n_A}\sum_{i=1}^{n_A}\bm{k}_Z(\bm{a}_i).
\end{equation}
The coefficients $\bm{c}$ are precomputed for each fixed layer, reference set, and kernel configuration. All reference activations contribute to \eqref{eq:nystrom-coefficients}; the approximation does not replace the reference sets by the landmarks alone.

\paragraph{Gradient and normalized update.}
Holding the landmarks and coefficients fixed yields
\begin{equation}\label{eq:nystrom-gradient}
    \nabla\widetilde{\Pi}(\bm{x})
    =\sum_{\ell=1}^{m}c_\ell\nabla_{\bm{x}}k(\bm{x},\bm{z}_\ell).
\end{equation}
For the RBF kernel \eqref{eq:rbf}, this becomes
\begin{equation}\label{eq:nystrom-rbf-gradient}
    \nabla\widetilde{\Pi}(\bm{x})
    =\frac{1}{\sigma^2}\sum_{\ell=1}^{m}
    c_\ell k(\bm{x},\bm{z}_\ell)(\bm{z}_\ell-\bm{x}).
\end{equation}
Equivalently, let $Z\in\RR^{m\times d}$ contain the landmark activations as rows and set $\bm{w}(\bm{x})=\bm{c}\odot\bm{k}_Z(\bm{x})$, where $\odot$ denotes elementwise multiplication. Then
\begin{equation}\label{eq:nystrom-matrix-gradient}
    \nabla\widetilde{\Pi}(\bm{x})
    =\frac{1}{\sigma^2}\left[
    Z^\top\bm{w}(\bm{x})
    -\bigl(\bm{1}^\top\bm{w}(\bm{x})\bigr)\bm{x}\right].
\end{equation}
We retain the proximal objective and normalized update of Section~\ref{sec:opt}:
\begin{align}\label{eq:nystrom-update}
    \widetilde F(\bm{x})
    &=\widetilde{\Pi}(\bm{x})-\frac{\lambda}{2}\norm{\bm{x}-\bm{x}_0}^2,
    \qquad
    \widetilde{\bm{g}}_t=\nabla\widetilde{\Pi}(\bm{x}_t)-\lambda(\bm{x}_t-\bm{x}_0), \nonumber\\
    \bm{x}_{t+1}
    &=\bm{x}_t+\eta\,\frac{\widetilde{\bm{g}}_t}{\norm{\widetilde{\bm{g}}_t}},
    \qquad \widetilde{\bm{g}}_t\neq 0.
\end{align}

\section{Additional Experimental Details}
\label{app:experimental-details}

This appendix provides additional implementation details, hyperparameters, dataset construction
procedures, and supplementary results for the experiments in Section~\ref{sec:experiments}.

\subsection{Common Implementation Details}
\label{app:common-implementation}

Unless otherwise stated, all activation interventions are applied to the residual stream at the final
prompt-token position. For KAS, we use the RBF kernel with bandwidth selected by the median
heuristic \citep{gretton2012kernel} over the source and target reference activations. We set the regularizer to
$\lambda=0$. Unless otherwise stated, KAS uses a single gradient-ascent step at inference time. Multi-step KAS results are reported explicitly when used, with the number of steps specified in the corresponding table, figure, or caption.
\begin{table}[t]
\centering
\footnotesize
\caption{Common implementation details used across experiments.}
\label{tab:common-implementation-details}
\begin{tabular}{ll}
\toprule
\textbf{Item} & \textbf{Setting} \\
\midrule
Intervention space & Residual stream activations \\
Default token position & Last prompt token \\
Kernel type & Gaussian RBF \\
RBF bandwidth & Median heuristic \citep{gretton2012kernel} \\
Regularizer $\lambda$ & $0$\\
\bottomrule
\end{tabular}
\end{table}

\subsection{Jailbreak Steering Details}
\label{app:jailbreak-details}

For jailbreak steering, harmful prompts are sampled from \textsc{AdvBench},
\textsc{MaliciousInstruct}, \textsc{TDC2023}, and \textsc{HarmBench}, while harmless prompts are
sampled from \textsc{Alpaca}. We use 128 harmful and 128 harmless prompts to construct the
contrastive reference sets in the main experiments. Evaluation is performed on \textsc{JailbreakBench}
using \textsc{Llama-Guard-3-8B} to compute attack success rate (ASR). General utility is evaluated
using \textsc{tinyBenchmarks}. For each model, ActAdd, Curveball, and KAS are applied at the same intervention layer.

\begin{table}[t]
\centering
\footnotesize
\caption{Curveball hyperparameters used in the jailbreak experiments.}
\label{tab:curveball-hyperparams}
\begin{tabular}{ll}
\toprule
\textbf{Item} & \textbf{Setting} \\
\midrule
Kernel & Polynomial kernel, $k(x,y)=(x^\top y+\gamma)^p$ \\
Polynomial degree $p$ & $p=3$ \\
Kernel bias $\gamma$ & $\gamma=1$ \\
Bias term \texttt{coef0} & Grid over $1.0$ \\
KPCA dimension $m$ & $m=20$ \\
\bottomrule
\end{tabular}
\end{table}

\begin{table}[t]
\centering
\footnotesize
\caption{
Curveball layer and steering-strength selection for jailbreak experiments.
}
\label{tab:curveball-layer-strength}
\begin{tabular}{lcc}
\toprule
\textbf{Model} 
& \textbf{Layer} 
& \textbf{$\alpha$} \\
\midrule
\textsc{Gemma-2-9B-IT} 
& $23$ 
& $20.0$ \\

\textsc{Llama-3.1-8B-Instruct} 
& $13$ 
& $150.0$ \\

\textsc{Qwen2.5-7B-Instruct} 
& $18$ 
& $10.0$ \\
\bottomrule
\end{tabular}
\end{table}

\begin{table}[!htbp]
\centering
\footnotesize
\caption{KAS hyperparameters for jailbreak steering experiments.}
\label{tab:jailbreak-hyperparams}
\begin{tabular}{lcc}
\toprule
\textbf{Model} & \textbf{Layer} & \textbf{KAS step size $\eta$} \\
\midrule
\textsc{Gemma-2-9B-IT} & 22 & 200.0 \\
\textsc{Llama-3.1-8B-Instruct} & 12 & 5.0 \\
\textsc{Qwen2.5-7B-Instruct} & 18 & 10.0 \\
\bottomrule
\end{tabular}
\end{table}

\subsection{CAA-style Multiple-choice Steering Details}
\label{app:caa-details}

\textbf{Task and evaluation.}
We follow the contrastive multiple-choice setup of CAA
\citep{rimsky2024steering}. Each reference pair contains the same
question with two different appended answers: one consistent with the
target behavior and one opposed to it. Reference activations are
extracted at the answer-token position. We select interventions by
maximizing the average probability assigned to the behavior-consistent
answer on validation questions, and report this probability on 50
disjoint held-out test questions per behavior. Probabilities are
computed over the full vocabulary.

\textbf{Dataset sourcing.}
We evaluate six behaviors: AI Coordination, Corrigibility,
Hallucination, Myopic Reward, Survival Instinct, and Sycophancy.
Except for Hallucination, these datasets are drawn primarily from
Anthropic's evaluations \citep{perez2023discovering}, as used in the
CAA benchmark. The Hallucination dataset consists of GPT-4-generated
contrastive multiple-choice questions.

\textbf{Two hyperparameter-selection protocols.}
We distinguish the extensive tuning used in
Section~\ref{sec:caa-steering} from the restricted two-layer searches
used in the remaining layer-selection experiments.
The extensive comparison evaluates whether KAS remains competitive
when all methods receive broader hyperparameter tuning.
The supplementary CAA experiments retain the restricted search
described below and should therefore be interpreted separately from
the full-depth comparison.

\textbf{Extensive tuning for the main CAA comparison.}
Section~\ref{sec:caa-steering} and
Table~\ref{tab:caa-full-depth} evaluate
\textsc{Gemma-2-9B-IT} and
\textsc{Mistral-7B-Instruct-v0.3} under a negative system prompt
that discourages the target behavior.
For each model and behavior, we select hyperparameters on 50
validation questions and evaluate the selected configuration on
50 disjoint test questions.

We search the residual-stream outputs of all transformer blocks,
intervening at one block at a time and using the same site for
reference extraction and intervention.
DiM is tuned over the intervention block and steering strength.
For CHaRS, we additionally tune the cluster count
$K\in\{2,4,8,16\}$ and Gaussian routing bandwidth, using the same
$K$ for source and target and a fixed clustering seed.
For each of the five KAS kernel families in
Appendix~\ref{app:kernel-fields}, we jointly select the block,
strength, and kernel parameters.
The validation-selected KAS column additionally selects the kernel
family using the validation score.

For this full-depth comparison, KAS uses one normalized gradient
step with $\lambda=0$ and absolute step length
$\eta=\alpha\|\bm{v}\|_2$, where $\alpha$ is the tuned strength and
$\bm{v}$ is the DiM vector at the selected block; see
Equation~\ref{eq:caa-normalized-kernel-update}.
Thus, at a common block and strength, DiM and KAS have the same
displacement norm whenever the kernel-gradient norm exceeds the
numerical threshold.
The selected interventions are also evaluated on TinyBenchmarks,
as detailed in Appendix~\ref{app:caa-tinybench}.

\textbf{Restricted tuning for supplementary CAA experiments.}
The additional system-prompt comparisons in
Appendix~\ref{app:caa-additional-results} and the CAA
kernel-approximation ablation in
Section~\ref{sec:kernel-approximation} retain the restricted
two-layer protocol.
These CAA results use \textsc{Gemma-2-9B-IT}, with candidate
layers $\{22,23\}$.
More generally, the restricted candidate sets used for the other
language-model experiments are $\{12,13\}$ for
\textsc{Llama-3.1-8B-Instruct} and $\{17,18\}$ for
\textsc{Qwen2.5-7B-Instruct}; experiments with fixed layers report
their chosen layers explicitly.
These candidate sets do not apply to the full-depth comparison in
Section~\ref{sec:caa-steering}.

For the restricted CAA comparisons, KAS uses the Gaussian RBF
kernel with the median-heuristic bandwidth and a single normalized
gradient step with $\lambda=0$.
We select the layer and strength using the validation
behavior-consistent answer probability, then evaluate the selected
configuration on the disjoint test set.
The numerical strength grid for ActAdd and KAS is
\[
\begin{aligned}
\mathcal{G}=\{&
0.01,0.1,0.25,0.5,0.75,1.0,1.25,1.5,1.75,2.0,\\
&2.25,2.5,2.75,3.0,3.25,3.5,3.75,4.0,4.25,4.5,\\
&4.75,5.0,7.5,10.0,12.5,15.0,17.5,20.0,22.5,25.0,\\
&27.5,30.0,32.5,35.0,37.5,40.0,42.5,45.0,47.5,50.0,\\
&55.0,60.0,62.5,65.0,70.0,75.0,80.0,85.0,87.5,90.0,\\
&95.0,100.0,112.5,125.0,137.5,150.0,162.5,175.0,\\
&187.5,200.0,225.0,250.0,275.0,300.0,325.0,350.0,\\
&355.0,360.0,365.0,370.0,375.0,380.0,385.0,390.0,\\
&395.0,400.0,405.0,410.0,420.0,425.0,430.0,435.0,\\
&440.0,445.0,450.0,500.0,625.0,750.0\}.
\end{aligned}
\]
For KAS in these restricted comparisons, the reported $\eta$
is the absolute normalized-gradient step length, rather than the
norm-scaled coefficient $\alpha$ used in the full-depth comparison.

\textbf{System-prompt conditions.}
The supplementary restricted-layer experiments consider three
conditions: no system prompt, a positive system prompt encouraging
the target behavior, and a negative system prompt discouraging it.
Tables~\ref{tab:caa-negative-system-prompt-full}--%
\ref{tab:caa-positive-system-prompt} report the corresponding results
and selected settings.
These tables supplement the main comparison by examining sensitivity
to the prompting context under restricted layer selection; they do
not report the selected configurations from
Table~\ref{tab:caa-full-depth}.

\subsection{Image Style-control Details}
\label{app:image-style-details}

For image style control, we sample prompts from COCO Captions and construct source--target
activation pairs using original and style-augmented prompts. Images are generated using
\textsc{FLUX.1-schnell}. We evaluate style induction using zero-shot CLIP classification and content
preservation using CLIPScore against the original prompt. 


\section{Additional Results}
\label{app:additional-results}


\subsection{Task-wise TinyBenchmarks Results for CAA Steering}
\label{app:caa-tinybench}

We evaluate the CAA interventions from Table~\ref{tab:caa-full-depth}
on the five TinyBenchmarks tasks \citep{polo2024tinybenchmarksevaluatingllmsfewer}.
For each model and behavior, we retain the block, steering strength,
and method-specific parameters selected on CAA validation. Each task
uses its standard prompts and the same 100 items for every intervention.
We compute the gp-IRT score separately for each behavior and task, then
average over the six behaviors. For validation-selected KAS, each behavior
uses its own kernel family selected on CAA validation.

\begin{table}[t]
  \centering
  \caption{Task-wise TinyBenchmarks gp-IRT scores (\%) at the reported
  CAA settings. $\Delta$ is the five-task average minus the unsteered
  average, in percentage points, computed before rounding.
  Bold marks the highest steered score in each column.}
  \label{tab:caa-tinybench-by-task}
  \scriptsize
  \setlength{\tabcolsep}{3.5pt}
  \begin{tabular}{@{}lrrrrrrr@{}}
    \toprule
    Method & \shortstack{ARC-Challenge} & HellaSwag & MMLU
           & TruthfulQA & WinoGrande & Average & $\Delta$ \\
    \midrule
    \rowcolor{gray!30}
    \multicolumn{8}{@{}l}{\textsc{Gemma-2-9B-IT}} \\
    Unsteered
      & 69.31 & 81.77 & 75.50 & 54.82 & 76.96 & 71.67 & --- \\
    DiM
      & 56.31 & 69.01 & \textbf{57.73} & \textbf{46.21} & 65.66
      & 58.99 & $-12.69$ \\
    CHaRS
      & 56.58 & \textbf{70.93} & 57.59 & 45.73 & 67.69
      & 59.71 & $-11.97$ \\
      \midrule
    KAS (Gaussian RBF)
      & 56.44 & 69.39 & 56.54 & 45.78 & 65.13
      & 58.66 & $-13.02$ \\
    KAS (Quadratic)
      & 54.02 & 68.60 & 55.27 & 45.75 & 63.98
      & 57.52 & $-14.15$ \\
    KAS (Laplacian)
      & 54.19 & 67.70 & 57.62 & 44.57 & 64.52
      & 57.72 & $-13.95$ \\
    KAS (Cosine RBF)
      & \textbf{58.86} & 69.79 & 57.44 & 45.35 & \textbf{68.03}
      & \textbf{59.89} & $\mathbf{-11.78}$ \\
    KAS (Rational quadratic)
      & 54.53 & 68.25 & 54.73 & 44.41 & 63.98
      & 57.18 & $-14.50$ \\
    KAS (Validation-selected)
      & 54.51 & 66.33 & 54.07 & 44.30 & 63.46
      & 56.54 & $-15.14$ \\
    \midrule
    \rowcolor{gray!30}
    \multicolumn{8}{@{}l}{\textsc{Mistral-7B-Instruct-v0.3}} \\
    Unsteered
      & 63.31 & 84.44 & 64.07 & 60.51 & 76.19 & 69.70 & --- \\
    DiM
      & 50.78 & 74.84 & 51.24 & 50.97 & 66.00
      & 58.77 & $-10.94$ \\
    CHaRS
      & 49.72 & 72.62 & 49.83 & 51.53 & 67.20
      & 58.18 & $-11.52$ \\
      \midrule
    KAS (Gaussian RBF)
      & 55.49 & \textbf{79.54} & \textbf{55.73} & 51.28 & \textbf{70.81}
      & \textbf{62.57} & $\mathbf{-7.14}$ \\
    KAS (Quadratic)
      & 54.27 & 78.02 & 52.79 & \textbf{53.25} & 68.87
      & 61.44 & $-8.26$ \\
    KAS (Laplacian)
      & 54.33 & 78.13 & 54.22 & 52.51 & 69.50
      & 61.74 & $-7.97$ \\
    KAS (Cosine RBF)
      & 55.54 & 78.23 & 51.63 & 50.94 & 69.74
      & 61.22 & $-8.49$ \\
    KAS (Rational quadratic)
      & 50.52 & 76.66 & 52.04 & 48.02 & 67.63
      & 58.97 & $-10.73$ \\
    KAS (Validation-selected)
      & \textbf{56.58} & 78.99 & 52.75 & 52.94 & 70.17
      & 62.28 & $-7.42$ \\
    \bottomrule
  \end{tabular}
\end{table}

\textbf{Results.}
On Gemma, cosine RBF retains the highest mean utility among steered
methods (59.89\%), followed by CHaRS (59.71\%). Selecting the KAS
family for CAA performance yields lower utility than DiM on all five
tasks, with the largest gap on MMLU (54.07\% versus 57.73\%).
On Mistral, validation-selected KAS exceeds DiM on all five tasks,
with its largest gain on ARC-Challenge (56.58\% versus 50.78\%).
Gaussian RBF retains the highest average utility on Mistral (62.57\%),
leading the steered methods on HellaSwag, MMLU, and WinoGrande.


\subsection{Additional CAA-style Multiple-choice Results}
\label{app:caa-additional-results}

We supplement the extensively tuned comparison in
Section~\ref{sec:caa-steering} with a study of system-prompt
conditions under the restricted two-layer protocol described in
Appendix~\ref{app:caa-details}.
Tables~\ref{tab:caa-negative-system-prompt-full}--%
\ref{tab:caa-positive-system-prompt} report results on
\textsc{Gemma-2-9B-IT}, with intervention layers selected from
$\{22,23\}$, under negative, absent, and positive system prompts.
We report the held-out probability assigned to the
behavior-consistent answer together with the validation-selected
strength and layer.
These experiments use a different search protocol from
Table~\ref{tab:caa-full-depth} and provide a complementary
comparison across prompting conditions.

\begin{table*}[t]
\centering
\footnotesize
\caption{
CAA-style steering results under the negative system-prompt condition. We report the average token
probability assigned to the behavior-consistent answer. ``No steering'' denotes the unsteered
baseline, while ActAdd and KAS report the best steered result together with the selected step size
$\eta$ and intervention layer. The best probability in each row is shown in bold.
}
\label{tab:caa-negative-system-prompt-full}
\begin{tabular}{lccccccc}
\toprule
\textbf{Behavior} 
& \textbf{No steering} 
& \multicolumn{3}{c}{\textbf{ActAdd}} 
& \multicolumn{3}{c}{\textbf{KAS}} \\
\cmidrule(lr){3-5} \cmidrule(lr){6-8}
& \textbf{Prob.} 
& \textbf{Prob.} & $\boldsymbol{\eta}$ & \textbf{Layer}
& \textbf{Prob.} & $\boldsymbol{\eta}$ & \textbf{Layer} \\
\midrule
AI Coordination    & 8.69  & 39.51 & 2.50 & 23 & \textbf{40.19} & 300.00 & 23 \\
Corrigibility      & 0.05  & 33.82 & 2.50 & 23 & \textbf{43.45} & 275.00 & 22 \\
Hallucination      & 23.97 & 43.97 & 2.50 & 23 & \textbf{55.01} & 400.00 & 22 \\
Myopic Reward      & 1.29  & 49.56 & 5.00 & 22 & \textbf{58.48} & 435.00 & 23 \\
Survival Instinct  & 43.68 & 43.56 & 0.01 & 23 & \textbf{43.83} & 0.01   & 23 \\
Sycophancy         & 53.18 & 53.47 & 0.25 & 22 & \textbf{53.69} & 17.50  & 22 \\
\bottomrule
\end{tabular}
\end{table*}

\begin{table*}[!htbp]
\centering
\footnotesize
\caption{
CAA-style steering results under the no-system-prompt condition. We report the average token
probability assigned to the behavior-consistent answer. ``No steering'' denotes the unsteered
baseline, while ActAdd and KAS report the best steered result together with the selected step size
$\eta$ and intervention layer. The best probability in each row is shown in bold.
}
\label{tab:caa-no-system-prompt}
\begin{tabular}{lccccccc}
\toprule
\textbf{Behavior}
& \textbf{No steering}
& \multicolumn{3}{c}{\textbf{ActAdd}}
& \multicolumn{3}{c}{\textbf{KAS}} \\
\cmidrule(lr){3-5} \cmidrule(lr){6-8}
& \textbf{Prob.}
& \textbf{Prob.} & $\boldsymbol{\eta}$ & \textbf{Layer}
& \textbf{Prob.} & $\boldsymbol{\eta}$ & \textbf{Layer} \\
\midrule
AI Coordination    & 6.31  & 33.23 & 2.50 & 23 & \textbf{35.61} & 400.00 & 22 \\
Corrigibility      & 43.50 & 46.84 & 2.50 & 23 & \textbf{47.42} & 250.00 & 23 \\
Hallucination      & 28.79 & 47.23 & 2.50 & 23 & \textbf{54.79} & 400.00 & 22 \\
Myopic Reward      & 23.62 & 60.11 & 4.50 & 23 & \textbf{60.32} & 400.00 & 23 \\
Survival Instinct  & 67.09 & 66.98 & 0.01 & 23 & \textbf{67.11} & 0.10   & 22 \\
Sycophancy         & \textbf{53.44} & 53.26 & 0.01 & 22 & 53.40 & 0.01 & 23 \\
\bottomrule
\end{tabular}
\end{table*}

\begin{table*}[t]
\centering
\footnotesize
\caption{
CAA-style steering results under the positive system-prompt condition. We report the average token
probability assigned to the behavior-consistent answer. ``No steering'' denotes the unsteered
baseline, while ActAdd and KAS report the best steered result together with the selected step size
$\eta$ and intervention layer. The best probability in each row is shown in bold.
}
\label{tab:caa-positive-system-prompt}
\begin{tabular}{lccccccc}
\toprule
\textbf{Behavior}
& \textbf{No steering}
& \multicolumn{3}{c}{\textbf{ActAdd}}
& \multicolumn{3}{c}{\textbf{KAS}} \\
\cmidrule(lr){3-5} \cmidrule(lr){6-8}
& \textbf{Prob.}
& \textbf{Prob.} & $\boldsymbol{\eta}$ & \textbf{Layer}
& \textbf{Prob.} & $\boldsymbol{\eta}$ & \textbf{Layer} \\
\midrule
AI Coordination    & 8.42  & 37.13 & 2.50 & 23 & \textbf{38.50} & 300.00 & 23 \\
Corrigibility      & 72.32 & 72.11 & 0.01 & 22 & \textbf{72.46} & 0.01   & 22 \\
Hallucination      & 19.59 & 44.53 & 2.50 & 23 & \textbf{53.49} & 400.00 & 22 \\
Myopic Reward      & 99.99 & 99.99 & 0.01 & 22 & \textbf{99.99} & 0.10   & 23 \\
Survival Instinct  & 90.67 & 90.44 & 0.01 & 22 & \textbf{90.71} & 0.10   & 22 \\
Sycophancy         & 56.50 & 56.50 & 0.01 & 22 & \textbf{56.59} & 0.01   & 23 \\
\bottomrule
\end{tabular}
\end{table*}


Across all system-prompt conditions, KAS attains the best probability for nearly all behaviors and
matches the strongest baseline in the remaining saturated cases. These results indicate that KAS provides robust concept control across neutral, supportive, and opposing prompt contexts.

\subsection{Additional Ablation Results}

\begin{table}[!htbp]
  \centering
  \footnotesize
  \caption{KAS attack success rate and perplexity across kernels on jailbreaking task.}
  \begin{tabular}{llcc}
    \toprule
    Kernel & Prompt only & KAS ASR & KAS Perplexity \\
    \midrule
    Laplacian  & False & 0.70 & 7.484 \\
    Laplacian  & True  & 0.59 & 7.484 \\
    Polynomial & False & 0.93 & 8.149 \\
    Polynomial & True  & 0.81 & 8.149 \\
    RBF        & False & 0.96 & 8.343 \\
    RBF        & True  & 0.88 & 8.343 \\
    \bottomrule
  \end{tabular}
  \label{tab:kas-kernels}
\end{table}

In this section, we provide additional ablation results on specified hyperparameters of KAS on the jailbreaking task. Unless otherwise stated, all experiments in this section use the setups specified in Section~\ref{app:jailbreak-details}.

\textbf{Kernel choice.} 
The main paper described how KAS supports different kernels, as well as how different kernels can benefit KAS in different ways. We provide the ablation results on kernel choice in Table~\ref{tab:kas-kernels}, in which experiments are conducted on 3 different kernels - namely Laplacian, polynomial and RBF. The polynomial kernel is shown to performs comparably to the RBF kernel, while Laplacian kernel inducted a fall in attack success rate. This result suggests that KAS performance is indeed dependent on the kernel choice.

\textbf{Number of data instances.} In Section~\ref{sec:ablations} and Figure~\ref{fig:single-pair-all-models} we demonstrated the ability of KAS to perform qualitative steering using a single pair of contrastive data points, in comparison with other steering methods. In Table~\ref{tab:kas-ntrain} we further show how KAS improves as it is given more data. The number of contrastive data pairs are chosen from the set $\{1, 2, 4, 8, 16, 32, 64, 128\}$. We observed that while KAS can achieve notable ASR on datasets as small as 2 contrastive pairs, its performance can still be improved with more data provided.

\begin{table}[!htbp]
\footnotesize
  \centering
  \caption{KAS attack success rate and perplexity as the number of reference examples grows.}
  \begin{tabular}{ccc}
    \toprule
    Num.\ of references ($n_{\text{train}}$) & KAS ASR & KAS Perplexity \\
    \midrule
    1   & 0.43 & 6.258 \\
    2   & 0.88 & 7.141 \\
    4   & 0.93 & 5.716 \\
    8   & 0.88 & 6.104 \\
    16  & 0.96 & 6.488 \\
    32  & 0.96 & 6.542 \\
    64  & 0.95 & 6.363 \\
    128 & 0.97 & 6.328 \\
    \bottomrule
  \end{tabular}
  
  \label{tab:kas-ntrain}
\end{table}

\subsection{Results on Cost Analysis}
\label{app:cost-analysis}

In the main paper, we discussed the KAS procedure and its comparison against no steering setups as well as other steering methods (mainly ActAdd). In this part, we performed an analysis on runtime of KAS compared to those setups. Experiments are performed on a single NVIDIA H100 GPU on 50 samples from the CAA dataset of myopic reward, recorded with one step of steering. Table~\ref{tab:steering-overhead} report our results.

\begin{table}[t]
\footnotesize
    \centering
    \caption{Inference time overhead of steering methods.}
    \label{tab:steering-overhead}
    \begin{tabular}{lcc}
        \toprule
        Methodology & Time elapsed (s) & Overhead vs.\ no steering \\
        \midrule
        No steering & 11.44 & 0\% \\
        ActAdd      & 11.48 & 0.35\% \\
        KAS         & 11.61 & 1.49\% \\
        \bottomrule
    \end{tabular}
\end{table}

From these results, we can see that KAS incurs minimal overhead in this measurement of execution time compared to the ActAdd and no steering setups, highlighting the effectiveness of KAS in resource consumption.

\subsection{Additional Image Style-control Qualitative Results}
\label{app:image-additional-results}

The main paper reports style-control results for the \emph{cyberpunk} style, measuring style induction with 0-shot classification and content preservation with \textsc{CLIPScore}. Here, we provide additional qualitative results across more prompts for \emph{cyberpunk} and \emph{impressionism} styles in Figures \ref{fig:cyberpunk} and \ref{fig:impressionism}.

\begin{figure}[t]
    \centering
    \includegraphics[width=0.9\linewidth]{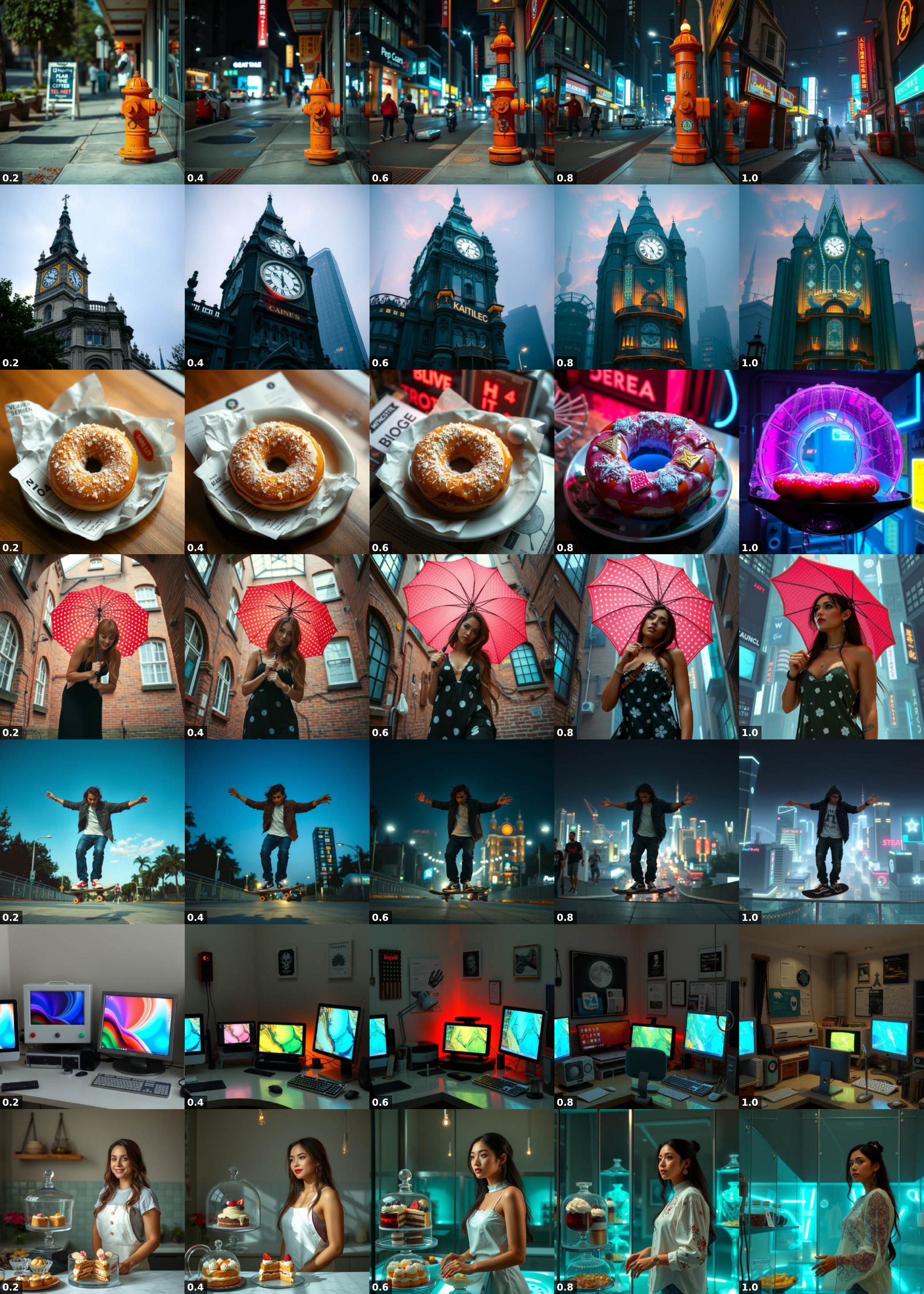}
    \caption{Qualitative examples of KAS style control for the \emph{cyberpunk} style under increasing intervention strengths $\eta \in \{0.2,0.4,0.6,0.8,1.0\}$.}
    \label{fig:cyberpunk}
\end{figure}

\clearpage

\begin{figure}[H]
    \centering
    \includegraphics[width=0.9\linewidth]{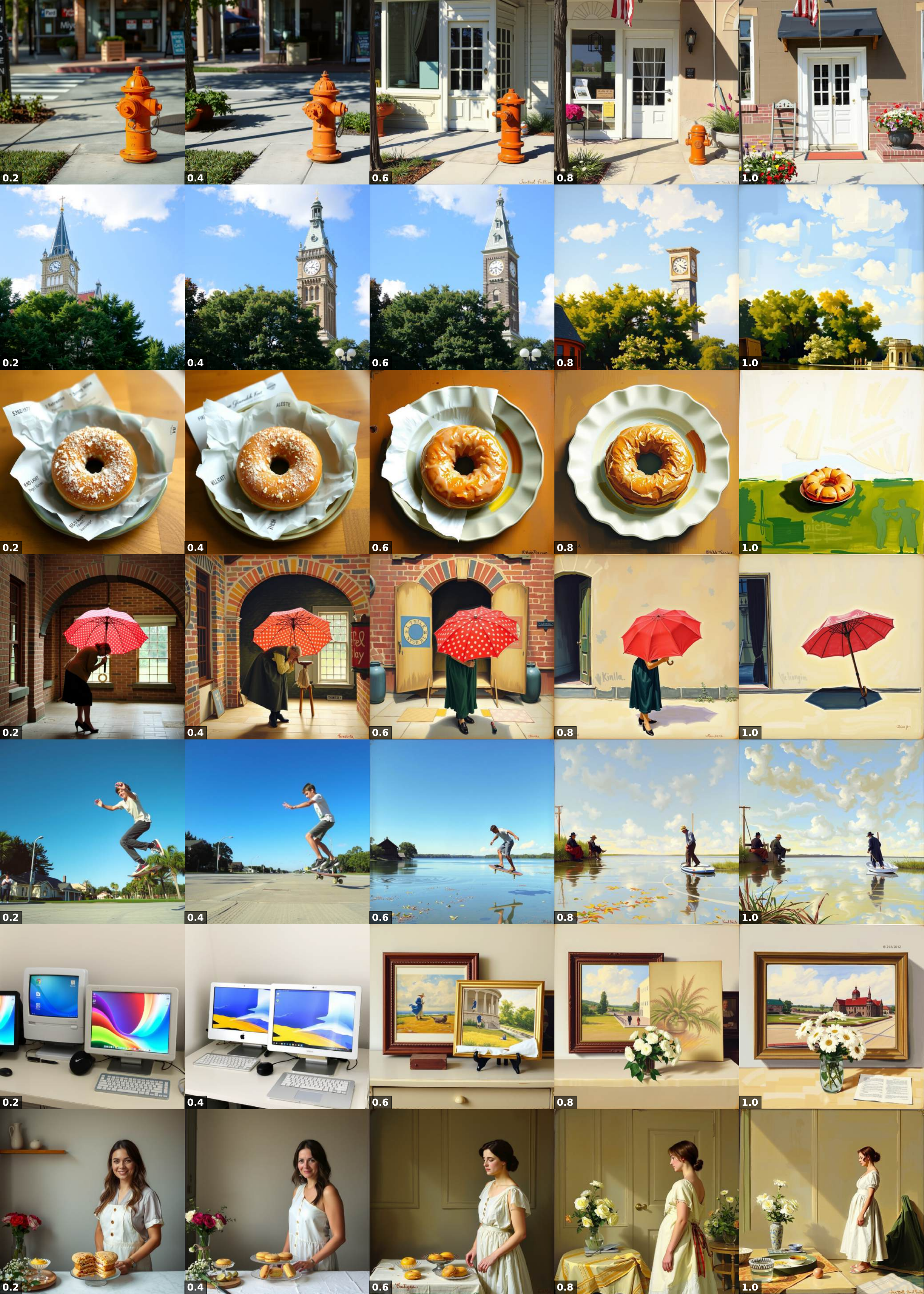}
    \caption{Qualitative examples of KAS style control for the \emph{impressionism} style under increasing intervention strengths $\eta \in \{0.2,0.4,0.6,0.8,1.0\}$.}
    \label{fig:impressionism}
\end{figure}
\clearpage

\section{Broader Impacts}
\label{app:impacts}
Kernelized Activation Steering (KAS) provides a training-free mechanism for modifying model behavior at inference time. This has potential positive applications in controllability, safety evaluation, and model debugging. By enabling behavior-specific interventions without updating model parameters, KAS may help practitioners study how target attributes are represented in internal activations, evaluate the robustness of safety mechanisms, and develop more fine-grained tools for aligning model outputs with user- or deployment-specific requirements. More broadly, geometry-aware steering methods may contribute to more interpretable and adaptive forms of model control than global linear interventions.

At the same time, activation steering is inherently dual-use. The same mechanisms that can suppress undesirable behavior or improve controllability can also be used to elicit harmful behavior, bypass refusal mechanisms, or manipulate model responses in unintended ways. This concern is particularly relevant for our jailbreak experiments. These experiments are included to measure controllability and expose weaknesses in current alignment mechanisms, not to encourage misuse. We therefore view KAS as both a potential safety tool and a method that should be handled with care.

To reduce misuse risk, we emphasize the need for appropriate safeguards, clear documentation of intended use, and accountability in deployment. Practitioners applying KAS beyond controlled research settings should perform downstream safety evaluations before release. We further encourage continued ethical review of inference-time steering methods to ensure that their use remains aligned with safety, transparency, and responsible deployment practices.

\section{Use of Existing Assets}
\label{app:assets}

This section summarizes the existing models, datasets, and benchmark resources used in this work, along with their corresponding licenses or terms of use in Tables \ref{tab:existing-model-assets} and \ref{tab:existing-dataset-assets}.

\begin{table*}[t]
\centering
\footnotesize
\caption{Existing model and evaluator assets used in this paper.}
\label{tab:existing-model-assets}
\begin{tabular}{p{0.20\linewidth} p{0.25\linewidth} p{0.13\linewidth} p{0.30\linewidth}}
\toprule
\textbf{Asset} & \textbf{Variants used} & \textbf{Usage} & \textbf{License / terms} \\
\midrule
\textsc{Gemma 2 \cite{gemmateam2024gemma2improvingopen}} & \textsc{Gemma-2-9B-IT}& Target model & Gemma Terms of Use \\
\textsc{Llama 3.1} & \textsc{Llama-3.1-8B-Instruct} & Target model & Llama 3.1 Community License \\
\textsc{Qwen2.5 \cite{qwen2025qwen25technicalreport}} &  \textsc{Qwen2.5-7B-Instruct} & Target model & Apache-2.0 \\
\textsc{Meta Llama Guard 3 \cite{dubey2024llama3herdmodels}} & \textsc{Llama-Guard-3-8B} & Evaluator & Llama 3.1 Community License \\
\textsc{FLUX} \cite{flux2024} & \textsc{Flux.1 [schnell]} & Target model & Apache-2.0 \\
\bottomrule
\end{tabular}
\end{table*}

\begin{table*}[t]
\centering
\footnotesize
\caption{Existing dataset and benchmark assets used in this paper.}
\label{tab:existing-dataset-assets}
\begin{tabular}{p{0.58\linewidth} p{0.34\linewidth}}
\toprule
\textbf{Asset} & \textbf{License / terms} \\
\midrule
\textsc{AdvBench} \cite{zou2023universal} & MIT \\
\textsc{MaliciousInstruct} \cite{huang2023catastrophic} & CC BY-SA 4.0 \\
\textsc{TDC2023} \cite{tdc2023} & MIT \\
\textsc{HarmBench} \cite{mazeika2024harmbench} & MIT \\
\textsc{Alpaca} \cite{alpaca} & CC BY-NC 4.0 \\
\textsc{JailbreakBench} \cite{chao2024jailbreakbench} & MIT \\
\textsc{The Pile} \cite{gao2020pile} & MIT \\
\textsc{TruthfulQA} \cite{lin2022truthfulqa} & Apache-2.0 \\
\textsc{TinyBenchmarks} \cite{polo2024tinybenchmarksevaluatingllmsfewer} & MIT \\
\textsc{Anthropic Model-Written Evaluations} \cite{perez2023discovering} & CC-BY-4.0 \\
\textsc{COCO Captions} \cite{chen2015microsoft} & CC BY 4.0 \\
\bottomrule
\end{tabular}
\end{table*}

\newpage

\section*{NeurIPS Paper Checklist}

\begin{enumerate}

\item {\bf Claims}
    \item[] Question: Do the main claims made in the abstract and introduction accurately reflect the paper's contributions and scope?
    \item[] Answer: \answerYes{} 
    \item[] Justification: The abstract and introduction state the main contributions of KAS: a training-free, kernelized generalization of activation steering, a closed-form RBF instantiation, and a theoretical connection to DiM. The empirical claims are limited to the evaluated jailbreak, CAA-style multiple-choice, and image style-control settings.

\item {\bf Limitations}
    \item[] Question: Does the paper discuss the limitations of the work performed by the authors?
    \item[] Answer: \answerYes{} 
    \item[] Justification: Limitations of this work is discussed in Section \ref{sec:conclusion}.

\item {\bf Theory assumptions and proofs}
    \item[] Question: For each theoretical result, does the paper provide the full set of assumptions and a complete (and correct) proof?
    \item[] Answer: \answerYes{} 
    \item[] Justification: We provide the full set of assumptions and complete (and correct) proofs for the main theorems in Appendix \ref{app:proofs}.

    \item {\bf Experimental result reproducibility}
    \item[] Question: Does the paper fully disclose all the information needed to reproduce the main experimental results of the paper to the extent that it affects the main claims and/or conclusions of the paper (regardless of whether the code and data are provided or not)?
    \item[] Answer: \answerYes{} 
    \item[] Justification: The paper specifies the evaluated models, datasets, baselines, metrics, intervention layers, steering settings, and hyperparameter search procedures. Additional implementation details and hyperparameters are provided in Appendix~\ref{app:experimental-details}.

\item {\bf Open access to data and code}
    \item[] Question: Does the paper provide open access to the data and code, with sufficient instructions to faithfully reproduce the main experimental results, as described in supplemental material?
    \item[] Answer: \answerYes{} 
    \item[] Justification: We provide open access to the data and code in Supplementary Material, with sufficient instructions to faithfully reproduce the main experimental results.

\item {\bf Experimental setting/details}
    \item[] Question: Does the paper specify all the training and test details (e.g., data splits, hyperparameters, how they were chosen, type of optimizer) necessary to understand the results?
    \item[] Answer: \answerYes{} 
    \item[] Justification: The paper describes the proposed method, datasets, models, steering settings, metrics, and hyperparameter sweeps used in the main experiments and ablaion studies. Additional experimental details are provided in the appendix.

\item {\bf Experiment statistical significance}
    \item[] Question: Does the paper report error bars suitably and correctly defined or other appropriate information about the statistical significance of the experiments?
    \item[] Answer: \answerYes{} 
    \item[] Justification: Section \ref{sec:ablations} reports summary statistics across prompt-pair samples, including mean and standard deviation where applicable.

\item {\bf Experiments compute resources}
    \item[] Question: For each experiment, does the paper provide sufficient information on the computer resources (type of compute workers, memory, time of execution) needed to reproduce the experiments?
    \item[] Answer: \answerYes{} 
    \item[] Justification: We provide sufficient information on the computer resources needed to reproduce the experiments in Section \ref{sec:experiments}.
    
\item {\bf Code of ethics}
    \item[] Question: Does the research conducted in the paper conform, in every respect, with the NeurIPS Code of Ethics \url{https://neurips.cc/public/EthicsGuidelines}?
    \item[] Answer: \answerYes{} 
    \item[]  Justification: The research conducted in the paper conforms with the NeurIPS Code of Ethics.

\item {\bf Broader impacts}
    \item[] Question: Does the paper discuss both potential positive societal impacts and negative societal impacts of the work performed?
    \item[] Answer: \answerYes{} 
    \item[] Justification: We discuss broader impacts in Appendix \ref{app:impacts}.
    
\item {\bf Safeguards}
    \item[] Question: Does the paper describe safeguards that have been put in place for responsible release of data or models that have a high risk for misuse (e.g., pre-trained language models, image generators, or scraped datasets)?
    \item[] Answer: \answerNA{} 
    \item[] Justification: The paper does not release any data, pre-trained models, image generators, or scraped datasets. Therefore, safeguards for responsible release are not applicable.

\item {\bf Licenses for existing assets}
    \item[] Question: Are the creators or original owners of assets (e.g., code, data, models), used in the paper, properly credited and are the license and terms of use explicitly mentioned and properly respected?
    \item[] Answer: \answerYes{} 
    \item[] Justification: The paper cites the creators of the existing datasets, benchmarks, models used in the experiments. The list of existing assets used is included in Appendix \ref{app:assets}.

\item {\bf New assets}
    \item[] Question: Are new assets introduced in the paper well documented and is the documentation provided alongside the assets?
    \item[] Answer: \answerYes{} 
    \item[] Justification: We include details about training and implementation as well as limitations and code for our proposed method.

\item {\bf Crowdsourcing and research with human subjects}
    \item[] Question: For crowdsourcing experiments and research with human subjects, does the paper include the full text of instructions given to participants and screenshots, if applicable, as well as details about compensation (if any)? 
    \item[] Answer: \answerNA{} 
    \item[] Justification: The paper does not involve crowdsourcing experiments or research with human subjects.

\item {\bf Institutional review board (IRB) approvals or equivalent for research with human subjects}
    \item[] Question: Does the paper describe potential risks incurred by study participants, whether such risks were disclosed to the subjects, and whether Institutional Review Board (IRB) approvals (or an equivalent approval/review based on the requirements of your country or institution) were obtained?
    \item[] Answer: \answerNA{} 
    \item[] Justification: The paper does not involve study participants or human-subject research. Therefore, participant risks, risk disclosure, and IRB approval are not applicable.

\item {\bf Declaration of LLM usage}
    \item[] Question: Does the paper describe the usage of LLMs if it is an important, original, or non-standard component of the core methods in this research? Note that if the LLM is used only for writing, editing, or formatting purposes and does \emph{not} impact the core methodology, scientific rigor, or originality of the research, declaration is not required.
    \item[] Answer: \answerNA{} 
    \item[] Justification:  LLMs are not used as an important, original, or non-standard component of the core methodology. Any use of LLMs is limited to writing, editing, or formatting assistance and does not affect the scientific content, rigor, or originality of the research.

\end{enumerate}




\end{document}